\documentclass[11pt]{amsart}
\usepackage{amsmath} 
\allowdisplaybreaks[4] 
\usepackage{epsfig}
\usepackage{ulem}
\usepackage{color}
\usepackage{verbatim}
\usepackage{graphicx,amssymb,amsmath,amsthm}
\usepackage{mathrsfs}
\usepackage{amssymb}
\usepackage{enumerate}
\usepackage{hyperref} 
\newcommand{\NN}{{\mathcal N}}

\newcommand{\R}{{\mathbb R}}

\renewcommand{\eqref}[1]{(\ref{#1})}
\newcommand{\inner}[1]{\langle #1 \rangle}

\newcommand{\cV}{{\mathcal V}}

\newcommand{\rank}{{\rm rank}}

\newcommand{\vx}{{\boldsymbol x}}

\newcommand{\vz}{{\boldsymbol z}}
\newcommand{\vw}{{\boldsymbol w}}

\newcommand{\vp}{{\boldsymbol w}}

\newcommand{\va}{{\boldsymbol a}}
\newcommand{\vn}{{\boldsymbol n}}

\newcommand{\vv}{{\boldsymbol v}}
\newcommand{\vu}{{\boldsymbol u}}

\renewcommand{\top}{T}
\newtheorem{prop}{Proposition}[section]
\newtheorem{lem}[prop]{Lemma}
\newtheorem{defi}{Definition}[section]
\newcommand{\innerp}[1]{\langle {#1} \rangle}

\newtheorem{theorem}{Theorem}[section]
\newtheorem{remark}{Remark}[section]
\newtheorem{question}{Question}[section]
\numberwithin{equation}{section}

 \usepackage{amsmath}

\begin{document}
\bibliographystyle{plain}


\title{ Finite Samples for  Shallow Neural Networks}

\author{Yu Xia}
\thanks{ Yu Xia was supported by NSFC grant (12271133, U21A20426, 11901143) and the key project of Zhejiang Provincial Natural Science Foundation grant (LZ23A010002)}
\address{Department of Mathematics, Hangzhou Normal University, Hangzhou 311121, China}
\email{yxia@hznu.edu.cn}

\author{Zhiqiang Xu}
\thanks{ Zhiqiang Xu is supported  by
the National Science Fund for Distinguished Young Scholars (12025108), NSFC (12471361, 12021001, 12288201) and National Key R\&D Program of China (2023YFA1009401). }
\address{ State Key Laboratory of Mathematical Sciences, Academy of Mathematics and Systems Science, Chinese Academy of Sciences, Beijing 100190, China\newline   School of Mathematical Sciences, University of Chinese Academy of Sciences, Beijing 100049, China.}
\email{xuzq@lsec.cc.ac.cn}

\maketitle

\begin{abstract}
This paper investigates the ability of finite samples to identify two-layer irreducible shallow networks with various nonlinear activation functions, including rectified linear units (ReLU) and analytic functions such as the logistic sigmoid and hyperbolic tangent. An ``irreducible" network is one whose function cannot be represented by another network with fewer neurons.
For ReLU activation functions, we first establish necessary and sufficient conditions for determining the irreducibility of a network. Subsequently, we prove a negative result: finite samples are insufficient for definitive identification of any irreducible ReLU shallow network. Nevertheless, we demonstrate that for a given irreducible network, one can construct a finite set of sampling points that can distinguish it from other network with the same neuron count.

Conversely, for logistic sigmoid and hyperbolic tangent activation functions, we provide a positive result. We construct finite samples that  enable the recovery of two-layer irreducible shallow analytic networks. To the best of our knowledge, this is the first study to investigate the exact identification of two-layer irreducible networks using finite sample function values.
Our findings provide insights into the comparative performance of networks with different activation functions under limited sampling conditions.
\end{abstract}
\section{Introduction}
\subsection{Problem Setup}
Assume that  $f: \mathbb{R}^{d} \rightarrow \mathbb{R}$ is selected from a function set $V$. 
A key challenge resides in distinguishing or reconstructing $f$ given a priori information about $V$.
 For instance, when  $V$ is regarded as a finite-dimensional function space such as a polynomial space or a spline space, there exists a collection of functions $\{g_1, \ldots, g_n\} \subseteq V$ such that any function $f \in V$ can be expressed as below:
\begin{equation}\label{eq:fcv} f = \alpha_1\cdot g_1 + \alpha_2\cdot  g_2 +  \cdots + \alpha_n\cdot g_n, \end{equation}
where $\alpha_j \in \mathbb{R}$, $j = 1, \ldots, n$, are the coefficients of such linear combination.
It is evident that the coefficient vector $(\alpha_1, \ldots, \alpha_n)\in \mathbb{R}^n$ corresponding to the function $f$ is unique if $g_1, \ldots, g_n$ form a basis for $V$. In this case, the task of recovering $f \in V$ is equivalent to identifying the coefficient vector $(\alpha_1, \ldots, \alpha_n) \in \mathbb{R}^n$. 

In the context of polynomial or spline spaces, a widely used approach for uniquely determining the coefficient vectors is interpolation from a finite set of samples. 
 For a general function set $V$, a natural question arises: 
 
\textit{Does there exist a finite set of samples $\{\vx_j\}_{j=1}^n \subseteq \mathbb{R}^d$ such that one can uniquely determine any $f \in V$ from the collection 
of input-output pairs $\{(\vx_j, f(\vx_j))\}_{j=1}^n$?}

In this paper, we focus on the case where $V$ represents a set of two-layer shallow neural networks.
 Neural networks have, in recent years, achieved remarkable success across various domains \cite{weinan, Hinton2012, background1, background2}. The efficacy of neural networks in approximating functions has been rigorously established over the past three decades \cite{UAP2, UAP1, UAP3,UAP, shen, shen1, Zhou1}. Specifically, some foundational works \cite{UAP2, UAP1} demonstrate that a continuous function defined on a bounded domain can be approximated by a sufficiently large shallow neural network. Notably, a considerable amount of research continues to examine the behavior of shallow networks, as demonstrated in the works referenced in \cite{2024new1,2024new2,2024new3}.  
 Although there exists a substantial corpus of theoretical research on the function approximation capabilities of neural networks, the question of how many samples are required to uniquely identify a neural network $f$ remains comparatively understudied from a theoretical perspective.

This paper seeks to address the aforementioned question within the context of two-layer shallow network.  Set
\begin{equation}\label{eq:para}
\mathcal{N}:=(\va_1,\ldots,\va_m)\times (b_1,\ldots,b_m)\times (s_1,\ldots,s_m)\times c \in \R^{d\times m}\times \R^m \times \R^m \times \R.
\end{equation}
 The two-layer shallow network $f_{\mathcal{N}}:\mathbb{R}^{d}\rightarrow \mathbb{R}$, parameterized by $\mathcal{N}$ in (\ref{eq:para}), can be expressed as
\begin{equation}\label{eqn: f_N}
f_{\mathcal{N}}(\vx)=\sum_{k=1}^{m}s_{k}\cdot \sigma(\langle\va_{k},\vx\rangle+b_k)+c.
\end{equation}
Here the function $\sigma( \cdot)$ is taken as some nonlinear activation function, such as the Rectified Linear Unit (ReLU) defined by $ \text{ReLU}(x) := \max\{x, 0\}$, the logistic sigmoid function $\text{Sigmoid}(x) := (1 + \exp(-x))^{-1}$, or the hyperbolic tangent function $ \tanh(x) := \frac{\exp(x) - \exp(-x)}{\exp(x) + \exp(-x)}$.
Building upon these concepts, two-layer shallow neural network set $V$ can be taken as:
\begin{equation}\label{eqn: V_shallow}
\begin{aligned}
V:=V_\sigma:=\big\{f_{\mathcal{N}}\ :\ &f_{\mathcal{N}}(\vx)=\sum_{k=1}^{m}s_{k}\cdot \sigma(\langle\va_{k},\vx\rangle+b_k)+c,\ \\ 
&\text{for}\ \ m\in \mathbb{Z}_{+},
\text{where $\mathcal{N}$ is defined in (\ref{eq:para})} 
\big\}.
\end{aligned}
\end{equation}
Throughout this paper, we  omit the subscript $\sigma$ in $V_\sigma$ for brevity,
 as the specific activation function employed in $V$ can be readily inferred from the context.

Moreover, we define  $\mathcal{N} \sim \mathcal{N}'$ as follows:

\begin{defi}
Two neural networks $f_\mathcal{N}\in V$ and $f_\mathcal{N'}\in V$
  are said to be equivalent, denoted as $\mathcal{N} \sim \mathcal{N}'$, if and only if they produce identical outputs for all possible inputs. Formally, this can be expressed as:
\[
f_{\mathcal{N}}(\boldsymbol{x}) = f_{\mathcal{N}'}(\boldsymbol{x}), \quad \text{for all } \boldsymbol{x} \in \mathbb{R}^{d}.
\]
 \end{defi}
The parameter $m$ in Equation (\ref{eqn: f_N}) is referred to as {\em the number of neurons} in the network represented by $f_{\mathcal{N}}$.
 The number of neurons in two networks  $f_{\mathcal{N}},f_{\mathcal{N}'} \in V$ may differ substantially, even when $f_{\mathcal{N}}(\vx)=f_{\mathcal{N}'}(\vx)$
for all $\vx\in \mathbb{R}^{d}$. For instance, consider $f_{\mathcal{N}}$ and $f_{\mathcal{N}'}$ as below:
\begin{equation}\label{eq:fnnc}
f_{\mathcal{N}}(\vx) = \sum_{k=1}^2 \left( \sigma(\langle \va_k, \vx \rangle) - \sigma(\langle -\va_k, \vx \rangle) \right)\  \text{and}\ f_{\mathcal{N}'}(\vx) = \sigma \Big( \big\langle \sum_{k=1}^2 \va_k, \vx \big\rangle \Big) - \sigma \Big( \big\langle -\sum_{k=1}^2 \va_k, \vx \big\rangle \Big).
\end{equation}
 Here, $\va_1, \va_2 \in \mathbb{R}^d$ and $\sigma(x) = \text{ReLU}(x)$. 
  By directly applying the identity $\sigma(x)\equiv x+\sigma(-x)$,
  a simple calculation shows that $f_{\mathcal N}$ in (\ref{eq:fnnc})
  employs $4$ neurons, whereas $f_{\mathcal{N}'}$ in (\ref{eq:fnnc}) utilizes only 2,  while $f_{\mathcal{N}}(\vx)=f_{\mathcal{N}'}(\vx)$ for all $\vx\in \mathbb{R}^d$.  Since $f_{\mathcal{N}'}$
 accomplishes the same function with fewer neurons, it represents a more efficient and desirable configuration. 
 
  A shallow network $f_{\mathcal{N}}$ is defined as {\em irreducible} if there does not exist another shallow network $f_{\mathcal{N}'}$ such that $\mathcal{N} \sim \mathcal{N}'$ and contains fewer neurons than $f_{\mathcal{N}}$.  Given that irreducible networks are generally more favorable, we will reformulate the aforementioned question within a more rigorous mathematical framework:
 \begin{question}\label{Q3}
  Let  $m\geq 2$ be a fixed positive integer.
Does there exist a finite set of points $\{\vx_j\}_{j=1}^n \subseteq \mathbb{R}^d$ such that for any two irreducible networks $f_{\mathcal{N}}\in V$ and $f_{\mathcal{N}'}\in V$ with $m$ neurons, the condition
\[
f_{\mathcal{N}}(\vx_j) = f_{\mathcal{N}'}(\vx_j), \quad j = 1, \ldots, n,
\]
implies that $f_\mathcal{N}(\vx)=f_{\mathcal{N}'}(\vx)$ for all $\vx\in \mathbb{R}^{d}$?
In other words, we seek to determine if there exists a finite set of points that can uniquely determine any irreducible network, when the number of neurons is given.
\end{question}

  In Question \ref{Q3}, we impose the constraint that the neural networks  $f_{\mathcal N}$
  and $f_{\mathcal N'}$    must have the same number of neurons. 
   This requirement is essential, as it is impossible for $f_\mathcal{N}(\vx)=f_{\mathcal{N}'}(\vx)$ to hold for all $\vx\in \mathbb{R}^{d}$ when irreducible networks $f_{\mathcal N}$
  and $f_{\mathcal N'}$  have different neuron counts. If such an equality were to occur,  it would imply that either $f_\mathcal{N}$
  or $f_\mathcal{N'}$  is not irreducible, contradicting our initial assumption of irreducibility for both networks.

\subsection{Our Contribution}
Given the variations in differentiability among activation functions, we categorize our analysis into two main classes: (1) shallow ReLU networks, and (2) shallow analytic networks. In the latter category, we jointly investigate the analytic activation functions $\sigma(x) = \text{Sigmoid}(x)$ and $\sigma(x) = \tanh(x)$. 
\subsubsection{Shallow ReLU Networks}  First, we present the necessary and sufficient conditions under which the number of neurons can be reduced, as outlined in Theorem \ref{th:minm}. In Theorem \ref{th:impos}, we provide a negative response to Question \ref{Q3} regarding shallow ReLU networks. Specifically, given the number of neurons $m$,  for any integer $n \in \mathbb{Z}_{+}$ and for any given point set $\{\vx_{j}\}_{j=1}^{n} \subseteq \mathbb{R}^{d}$, there exist two distinct irreducible networks $f_{\mathcal{N}}$ and $f_{\mathcal{N}'}$ that possess an identical number of neurons $m$, such that $f_{\mathcal{N}}(\vx_{j}) = f_{\mathcal{N}'}(\vx_{j})$ for $j = 1, \ldots, n$, while ensuring that $\mathcal{N} \not\sim \mathcal{N}'$. 
Although a finite point set proves insufficient for definitively identifying shallow ReLU networks, 
we demonstrate in Theorem \ref{th:samplingc} that it is possible to construct a finite set of sampling points tailored to a specific irreducible network \( f_{\mathcal{N}} \), which can distinguish it from other networks with the same number of neurons.
\subsubsection{Shallow Analytic Networks}
We offer a positive response to Question \ref{Q3} in Theorem \ref{th:gsamp}, demonstrating that it is indeed feasible to generate a finite set of points that can effectively differentiate irreducible shallow analytic networks. In contrast to the negative result obtained in shallow ReLU networks, this underscores a fundamental divergence in the behavior of shallow ReLU networks compared to shallow analytic networks with respect to finite point identification.

\subsection{Related Works}

\subsubsection{Irreducible Property of Neural Network}
The irreducible concept was first introduced in \cite{sussmann} for real-valued shallow networks employing the activation function $\sigma(x) = \tanh(x)$. 
Specifically, let $f_{\mathcal{N}}$ be represented as in (\ref{eqn: f_N}) with $\sigma(x) = \tanh(x)$, in \cite{sussmann}, Sussmann showed that $f_{\mathcal{N}}$ is irreducible if and only if none of the following conditions hold:
\begin{enumerate}[(i)]
\item One of the $ s_k$, $ k = 1, \ldots, m$, vanishes;
\item There exist two different indices $ k_1,k_2 \in\{1,\ldots,n\}$ such that $ |\sigma(\langle \va_{k_1}, \vx \rangle+ b_{k_1})| = |\sigma(\langle \va_{k_2}, \vx \rangle+ b_{k_2})| $ for all $ \vx \in \mathbb{R}^d$;
\item One of the  $ \sigma(\langle \va_k,   \cdot \rangle+ b_k)$, $ k = 1, \ldots, m $,  is a constant.
\end{enumerate}
 For complex-valued networks, analogous results can be found in \cite{kob10}.
When using the activation function $\sigma(x) = \text{ReLU}(x)$, Dereich and Kassing \cite{minimal_network} proved that for an irreducible network $f_{\mathcal{N}}$ as defined in (\ref{eqn: f_N}), if $\R^d$ is partitioned by $n$ hyperplanes into regions where $f_{\mathcal N}$ is affine, then the network $f_{\mathcal N}$ can only contain $n$, $n + 1$, or $n + 2$ neurons.
 However, the findings presented in \cite{minimal_network} do not offer a direct methodology for assessing whether the number of neurons in a specific ReLU network $f_{\mathcal{N}}$ can be reduced.

\subsubsection{Interpolation for Shallow Networks}
A related concept is \textit{interpolation for shallow networks}. Specifically, given $ n $ distinct points $\vx_i\in \mathbb{R}^{d}$, $i=1,\ldots,n$, and their associated target values $y_i\in \mathbb{R}$,  $i=1,\ldots,n$, the goal is to find the parameters $ \{\va_k\}_{k=1}^{m} $, $ \{b_k\}_{k=1}^{m} $, and $ \{s_k\}_{k=1}^{m} $ such that the following system of equations holds true:
\begin{equation}\label{eqn: interpolation}
\sum_{k=1}^{m} s_k\cdot   \sigma(\langle \va_k, \vx_i \rangle + b_k) = y_i, \quad i = 1, \ldots, n.
\end{equation}
Pinkus \cite{universal} demonstrated that for any $n$ distinct points $\vx_i\in \mathbb{R}^{d}$, $i=1,\ldots,n$, and their corresponding values $y_i\in \mathbb{R}$,  $i=1,\ldots,n$, one can always construct $ f_{\mathcal{N}} $ in the form of  (\ref{eqn: f_N}) with $ m=n $ using a non-polynomial activation function $ \sigma \in C(\mathbb{R}) $ that satisfies the interpolation requirement in (\ref{eqn: interpolation}).

It is important to note that a shallow network $f_{\mathcal{N}}$ with $m=n$
  satisfying the interpolation condition (\ref{eqn: interpolation}) is generally not unique. This can be illustrated through a simple example:
Given any $\vx_1 \in \mathbb{R}^d$  (where $d \geq 3$) and $y_1 \in \mathbb{R}$, we can construct two distinct irreducible shallow networks:
\[
f_{\mathcal{N}}(\vx) = \sigma(\langle \va, \vx \rangle) + y_1-\sigma(0) \quad \text{and} \quad f_{\mathcal{N}'}(\vx) = \sigma(\langle \va', \vx \rangle) + y_1-\sigma(0),
\]
where $\va\in \R^d$ and $\va'\in \R^d$ satisfy the conditions $\|\va\|_2=\|\va'\|_2=1$, $\langle \va, \vx_1 \rangle = \langle \va', \vx_1 \rangle = 0$, and $\va \neq \pm \va'$.
In this case, we have $f_{\mathcal{N}}(\vx_1) = f_{\mathcal{N}'}(\vx_1) = y_1$ and $m=n=1$, yet $\mathcal{N} \not\sim \mathcal{N}'$. This example demonstrates that when the number of neurons equals the number of interpolation points, the shallow network may not be uniquely determined by the interpolation conditions.

\subsubsection{Finite Samples for Identifiability}
There exists only a limited body of research addressing the issue of parameter identification in neural networks using finite samples. 
Existing studies on parameter identification through finite sampling, particularly for neural networks with analytic activation functions such as $\tanh( \cdot)$ and $\text{Sigmoid}( \cdot)$, predominantly rely on precise or approximate gradient estimation of $f_{\mathcal{N}}$, see as in 
  \cite{finite_analytic1,finite_analytic2,minimal_samples}.  

For ReLU networks, Rolnick and Kording \cite{Rolnick} introduced reverse engineering technique to construct finite samples for distinguishing ReLU networks; however, they did not provide an explicit number of samples required for exact recovery. Furthermore, Stock and Gribonval \cite{Stock} investigated the parameter identification problem for shallow ReLU networks within a bounded set. In \cite[Theorem 6]{Stock}, they demonstrated that for a given network \( f_\mathcal{N} \) that meets specific structural criteria, there exists a bounded set \( \mathcal{X} \subseteq \mathbb{R}^{d} \), constructed as a union of small balls, such that if \( f_{\mathcal{N}}(\mathbf{x}) = f_{\mathcal{N}'}(\mathbf{x}) \) for all \( \mathbf{x} \in \mathcal{X} \), then \( \mathcal{N}' \) is equivalent to \( \mathcal{N} \) up to permutation and scaling ambiguity. This result establishes identifiability using a bounded set. However, their work does not address finite samples for identifiability. Moreover, the relationship between the parameters of equivalent shallow ReLU networks extends beyond mere permutation and scaling ambiguity, as demonstrated in Theorem \ref{prop:ambiguity}.

Previous studies on neural network identifiability have encountered significant limitations. For networks with analytic activation functions without knowing the parameters, accurately estimating the gradient of $f_{\mathcal{N}}$ from finite samples poses considerable challenges. Regarding ReLU networks, existing research has primarily focused on bounded sets rather than finite sets. These constraints highlight the need for a more thorough investigation into the uniqueness of network determination via finite sampling.

\section{Main Results }

\subsection{Shallow ReLU Network}
In this subsection, the activation function is considered as $\sigma(x)= \text{ReLU}(x)$. 
A simple observation is that $\sigma(\lambda \cdot x)=\lambda\cdot \sigma(x)$ for any $\lambda\geq 0$. 
For convenience, for any fixed pair $(\va, b) \in \mathbb{R}^{d} \times \mathbb{R}$, the hyperplane $\mathcal{H}(\va, b)$ is defined as follows:
\begin{equation}\label{H}
 \mathcal{H}(\va, b) := \{\vx \in \mathbb{R}^{d} : \langle \va, \vx \rangle + b = 0\}, 
 \end{equation}
which will be frequently utilized in  subsequent discussions. When $d = 1$, $\mathcal{H}(\va, b)$ collapses to a single point if it is not empty.

 \subsubsection{Condition for Reducibility}

 Given that $m$ represents the number of neurons in $f_{\mathcal{N}}$, it is essential to identify the conditions under which $m$ can be reduced further. There are two trivial scenarios:
 \begin{enumerate}[(i)]
 \item  $s\cdot  \va = \boldsymbol{0}$: the equation $s\cdot \sigma(\left<\va, \vx\right> + b) + c = c'$ holds, with $c': = c + s\cdot \sigma(b)$;
  \item $(\va_{1}, b_{1}) = \lambda\cdot (\va_{2}, b_{2})$ for some $\lambda > 0$: the expression $s_1\cdot \sigma(\langle \va_1, \vx \rangle + b_{1}) + s_2\cdot  \sigma(\langle \va_2, \vx \rangle + b_{2})$ simplifies to $(s_{2} + \lambda \cdot  s_{1}) \cdot \sigma(\langle \va_2, \vx \rangle + b_2)$. 
 \end{enumerate}
Consequently, it is essential to introduce the concept of an \textit{admissible shallow ReLU network} to exclude the aforementioned two scenarios.
\begin{defi}[Admissible shallow ReLU network]
\label{def: irreducible}
The shallow network $f_{\mathcal{N}}$ defined in (\ref{eqn: f_N})  is said to be \textit{admissible} if it satisfies the following two conditions:
\begin{enumerate}[{\rm (i)}]
 \item For all $k \in \{1, \ldots, m\}$,  $s_k\cdot \va_k \neq \boldsymbol{0}$; 
 \item { For any constant $\lambda>0$ and any pair of distinct indices $k_1$
  and $k_2$  satisfying $1 \leq k_1 < k_2 \leq m$, the following inequality holds:
 \[
 (\va_{k_1}, b_{k_1}) \neq \lambda \cdot (\va_{k_2}, b_{k_2}).
 \]
 }
  \end{enumerate}
\end{defi}

In the ensuing discussions, unless explicitly stated otherwise, we will exclusively concentrate on admissible shallow ReLU networks.
The definition of admissible shallow ReLU networks does not preclude the possibility that
  there exist  $k_1, k_2 \in \{1, \ldots, m\}$ and $\lambda > 0$ such that $(\va_{k_2}, b_{k_2}) = -\lambda \cdot (\va_{k_1}, b_{k_1})$.  In this scenario, we observe that
    $\mathcal{H}(\va_{k_1},b_{k_1})=\mathcal{H}(\va_{k_2},b_{k_2})$ and 
\[
s_{k_1}\cdot  \sigma(\langle \va_{k_1}, \vx \rangle + b_{k_1}) +s_{k_2}\cdot \sigma(\langle \va_{k_2}, \vx \rangle + b_{k_2}) = s_{k_1}\cdot  \sigma(\langle \va_{k_1}, \vx \rangle + b_{k_1}) +(\lambda \cdot s_{k_2}) \cdot \sigma(\langle -\va_{k_1}, \vx \rangle - b_{k_1}).
\]
  Consequently, any admissible network $f_{\mathcal{N}}$ defined in (\ref{eqn: f_N}) can be equivalently expressed in the following form:
 \begin{equation}\label{eqn: fN_new}
f_{\mathcal{N}}(\vx)=\sum_{k\in K_1}\big(s_{k,1}\cdot \sigma(\langle \va_k,\vx\rangle+b_k)+s_{k,2}\cdot \sigma(\langle -\va_k,\vx\rangle-b_k)\big)+\sum_{k\in K_2}s_k\cdot \sigma (\langle\va_k,\vx\rangle+b_k)+c.
\end{equation}
Here, the hyperplanes in $\mathcal{H}(\va_k, b_k)$, ${k\in K_1\cup K_2}$,  are mutually distinct. The number of neurons $m$ in $f_{\mathcal N}$, as shown in equation (\ref{eqn: fN_new}), is given by 
$m=2\cdot \#K_1+\#K_2$,  where $\#K_1$
  and $\#K_2$  denote the cardinalities of  $K_1$
  and $K_2$, respectively.

We now present the necessary and sufficient conditions that allow for a further reduction in the number of neurons $m$.
  Recall that a shallow network $f_{\mathcal{N}}$ is said to be irreducible if the number of neurons cannot be reduced. For clarity, we refer to $f_{\mathcal{N}}$ as {\em reducible} if the number of neurons can indeed be reduced.
  
\begin{theorem}\label{th:minm}
Let $f_{\mathcal{N}}$
  be an admissible shallow ReLU network in the form of (\ref{eqn: fN_new}). 
  Then $f_{\mathcal{N}}$ is {reducible}
  if and only if one of the following three conditions is satisfied:
\begin{enumerate}[{\rm (i)}]
\item  $\# K_1=1$, and there exist   $\epsilon_k\in\{-1,+1\}$ for the sole $k\in K_1$,  and $K_2'\subseteq K_2$ such that 
\begin{equation}\label{cond1: m=1}
\sum_{k\in K_1}\epsilon_{k}\cdot s_{k,i_k}\cdot  {\va_k}+\sum_{k\in K_2'}s_k \cdot \va_k=\boldsymbol{0};
\end{equation}

\item  $\# K_1=2$, 
and there exist $\epsilon_k \in \{-1,+1\}$ for each $k \in K_1$, a subset $K_2' \subseteq K_2$, an index $k_0 \in K_1 \cup K_2$, and a constant $c_0 \in \mathbb{R}$
  such that
\begin{equation}\label{cond_m2}
{
\sum_{k\in K_1}\epsilon_{k}\cdot s_{k,i_k}\cdot {\va_k}+\sum_{k\in K_2'}s_k\cdot  \va_k+c_0\cdot \va_{k_0}=\boldsymbol{0};}
\end{equation}

\item $\# K_1\geq 3$. 
\end{enumerate}
For conditions {\rm (i)} and {\rm (ii)}, the index $i_k$
  is defined as: 
\begin{equation}\label{i_index}
i_k:=
\begin{cases}
1, &\epsilon_k=1;\\
2, &\epsilon_k=-1.
\end{cases}
\end{equation}

\end{theorem}
\begin{proof}
The proof can be found in Section \ref{sec_min}. 
\end{proof}
\begin{remark}
Let $n: = \# K_1 + \# K_2$  denote the number of distinct hyperplanes for $f_{\mathcal{N}}$ in the form of (\ref{eqn: fN_new}). A direct analysis of the non-differentiable set indicates that the number of neurons of  $f_{\mathcal{N}}$ must be at least $n$. According to \cite{minimal_network}, the number of neurons of irreducible $f_{\mathcal{N}}$  is constrained to $n$, $n + 1$, or $n + 2$.
Theorem \ref{th:minm} implies that for any irreducible $f_{\mathcal N}$, we have $\#K_1 \leq 2$. This finding aligns closely with the results on neuron count presented in \cite{minimal_network}. Specifically: 
\begin{enumerate}[{\rm (i)}]
\item When $\#K_1 = 0$, the number of neurons for irreducible $f_{\mathcal{N}}$  is $n$;
\item When $\#K_1 = 1$, the number of neurons for irreducible $f_{\mathcal{N}}$ is either $n$ or $n + 1$;
\item  When $\#K_1 \geq 2$, the number of neurons for irreducible $f_{\mathcal{N}}$ can be $n$, $n+1$, or $n + 2$.
\end{enumerate}
The crucial distinction and primary advantage of Theorem \ref{th:minm} lies in its provision of necessary and sufficient conditions for neuron reduction. While \cite{minimal_network} establishes bounds on the minimal number of neurons, it does not offer specific criteria for determining when the number of neurons in $f_{\mathcal{N}}$
  can be reduced.    In contrast, our theorem presents precise, verifiable conditions that comprehensively characterize the circumstances under which neuron reduction is feasible.
\end{remark}

\subsubsection{Finite Samples for Shallow ReLU Networks}

The following theorem offers a negative resolution to Question \ref{Q3} concerning shallow ReLU networks.
 \begin{theorem}\label{th:impos}
Let $n, m \in \mathbb{Z}_{+}$  be arbitrary fixed positive integers with $m \geq 2$.
For any fixed point set $\{\vx_{j}\}_{j=1}^{n} \subseteq \mathbb{R}^{d}$ $(\text{where }d\geq 2)$, there exist two corresponding irreducible shallow ReLU networks $f_{\mathcal{N}}$ and $f_{\mathcal{N}'} \in V$,   each composed of exactly $m$ neurons, such that
\[
f_{\mathcal{N}}(\vx_{j}) = f_{\mathcal{N}'}(\vx_{j}), \quad j = 1, \ldots, n,
\]
while also ensuring that $\mathcal{N} \not\sim \mathcal{N}'$, i.e., there exists some $\vx \in \mathbb{R}^{d}$ such that 
\[
f_{\mathcal{N}}(\vx) \neq f_{\mathcal{N}'}(\vx).
\]
\end{theorem}
\begin{proof}
The proof can be found in Section \ref{sec: negative_answer}. 
\end{proof}

Theorem \ref{th:impos} demonstrates that no finite point set can serve as a universal differentiator for all pairs of non-equivalent irreducible networks. Building upon this result, we now turn our attention to a closely related question:{ for a given irreducible $f_{\mathcal{N}}\in V$, is it possible to identify a finite point set associated with $f_{\mathcal{N}}$  that can distinguish it from all other irreducible $f_{\mathcal{N}'}\in V$?}

To formally articulate this question, we introduce the following definition:
\begin{defi}
A point set $\mathcal{X} \subseteq \mathbb{R}^d$  is said to be \textit{appropriate for an irreducible $f_{\mathcal{N}} \in V$} if it satisfies the following condition:
For any irreducible $f_{\mathcal{N}'} \in V$  with an equal number of neurons as $f_{\mathcal{N}}$, if
$f_{\mathcal{N}'}(\vx) = f_{\mathcal{N}}(\vx)  \text{ for all } \vx \in \mathcal{X},$
then it necessarily follows that $\mathcal{N}' \sim \mathcal{N}$.
\end{defi}
Now we can reformulate the question above as follows:
\begin{question}\label{Q3.1}
Given an irreducible $f_{\mathcal{N}} \in V$, does there exist a finite set of points $\mathcal{X} \subseteq \mathbb{R}^d$ that is appropriate for $f_{\mathcal{N}}$?
\end{question}

 We now provide a positive response to Question \ref{Q3.1} and specify the requisite number of points.
For $f_{\mathcal{N}}$  represented in (\ref{eqn: f_N}), the hyperplanes $\mathcal{H}(\va_k,b_k)$, $k=1,\ldots,m$,
are mutually distinct for almost all choices of $\{\va_k\}_{k=1}^{m} \subseteq \mathbb{R}^d$  and $\{b_{k}\}_{k=1}^m \subseteq \mathbb{R}$. Given this prevalent scenario, our analysis of finite samples primarily focuses on cases where these hyperplanes are mutually distinct. Notably, as established in Theorem \ref{th:minm}, this condition of distinct hyperplanes ensures that $f_{\mathcal{N}}$  is irreducible.
Here, the term ``almost all" denotes the exclusion of a set of measure zero, a concept that will  be pertinent in the subsequent discussions.

\begin{theorem}\label{th:samplingc}
 Let $f_{\mathcal{N}}: \mathbb{R}^d \to \mathbb{R}$
 be an irreducible shallow ReLU network with $m$ neurons, represented in the form of (\ref{eqn: f_N}). Assume that the hyperplanes $\mathcal{H}(\boldsymbol{a}_k, b_k)$, $k=1,\ldots,m$, are mutually distinct. Then, there exists a finite set of points $\mathcal{X} \subseteq \mathbb{R}^d$
  with cardinality $\#\mathcal{X} = (2m+2)\cdot m\cdot d$ that is appropriate for $f_{\mathcal{N}}$. 
\end{theorem}
\begin{proof}
The proof can be found in Section \ref{sec: finite_sampling_relu}.
\end{proof}

\begin{remark}
As we will see in Theorem \ref{prop:ambiguity} (presented in Section \ref{sec: prof_finite_relu}), the equivalence class of shallow ReLU networks encompasses more than just permutations and positive scalings. This broader equivalence class significantly complicates the construction and analysis required to achieve global identifiability, in contrast to results in \cite{Stock}.
\end{remark}
\begin{remark}
The explicit construction of the set $\mathcal{X} \subseteq \mathbb{R}^d$ which is  appropriate for $f_{\mathcal{N}}$
  can be found in the proof. While this construction yields a viable sampling set, it is conceivable that more efficient alternatives exist, potentially utilizing fewer points while still guaranteeing exact identification.
We would like to mention  that the number of sampling points  presented in  Theorem \ref{th:samplingc} is not optimal. We leave the determination of the minimal sufficient sample size as an open question for interested readers to explore. 
\end{remark}
\subsection{Shallow  Analytic Network}
This subsection extends the investigation to the cases  where $\sigma(x)=\text{Sigmoid}(x)$ or $\sigma(x)=\tanh(x)$. Since 
\begin{equation}\label{eqn: tanh_sigmoid}
\text{tanh}(x) \equiv 2\cdot\text{Sigmoid}(2x)-1,
\end{equation}
 they  can be analyzed jointly in this section. Moreover, we can directly get that the identity
\begin{equation}\label{eqn: c_0}
\sigma(x)+\sigma(-x)\equiv c_0,
\end{equation}
where  $c_0=1$ if $\sigma(x)=\text{Sigmoid}(x)$, and $c_0= 0$ if $\sigma(x)=\tanh(x)$. 

Likewise, there are  two trivial scenarios for reducing the number of neurons:
 \begin{enumerate}[(i)]
 \item  $s\cdot \va = \boldsymbol{0}$: the equation $s\cdot \sigma(\left<\va, \vx\right> + b) + c = c'$ holds, with $c' = c + s\cdot \sigma(b)$;
  \item $(\va_{1}, b_{1}) = \epsilon\cdot (\va_{2}, b_{2})$ for some $\epsilon\in\{-1,+1\}$: the expression 
  \[
  s_1\cdot \sigma(\langle \va_1, \vx \rangle + b_{1}) + s_2\cdot \sigma(\langle \va_2, \vx \rangle + b_{2})+c
  \]
   simplifies to 
   \[
   (s_{2} + \epsilon\cdot  s_{1})\cdot \sigma(\langle \va_2, \vx \rangle + b_2)+c+c_0 \cdot s_{1} \cdot \boldsymbol{1}_{\epsilon < 0}.
   \] 
 \end{enumerate}
   Here, $\boldsymbol{1}_{\epsilon < 0}$ is defined as:
\[
\boldsymbol{1}_{\epsilon < 0}=\begin{cases}
0, &\epsilon\geq 0;\\
1, &\epsilon<0.
\end{cases}
\]
Therefore,  it is imperative to provide the definition of an admissible network in the context of the analytic activation function $\sigma$. The following definition is applicable to both $\sigma(x) = \text{Sigmoid}(x)$ and $\sigma(x) = \tanh(x)$.
\begin{defi}[Admissible Shallow Analytic Network]
\label{def: irreducible_g} A neural network $f_{\mathcal{N}}$ as defined in (\ref{eqn: f_N}) 
is considered an \textit{admissible network under the activation function $\sigma=\mathrm{Sigmoid}(x)$ or $\sigma=\tanh(x)$}, if it satisfies the following two assumptions:
\begin{enumerate}[{\rm (i)}]
\item For all $k \in \{1, \ldots, m\}$,  $s_k\cdot \va_k \neq \boldsymbol{0}$; 
 \item {
 For all constant $\epsilon \in \{-1,+1\}$ and all indices $k_1, k_2$ such that $1 \leq k_1 < k_2 \leq m$, 
 \[
 (\va_{k_1}, b_{k_1}) \neq \epsilon\cdot (\va_{k_2}, b_{k_2}).
 \]}
\end{enumerate}
\end{defi}
Corollary 1 in \cite{sussmann} demonstrates that the admissibility condition stated above is equivalent to the irreducibility condition for shallow analytic networks when $\sigma(x) = \tanh(x)$. Based on (\ref{eqn: tanh_sigmoid}), this result can be readily generalized to the case when $\sigma(x) = \text{Sigmoid}(x)$. The conclusion is succinctly articulated in Theorem \ref{coro: minimal_analytic}. 
Furthermore,  Vla{\v c}i\'c and B\"olcskei have expanded the scope of this discourse to include general directed acyclic graphs and nonlinear activation functions $\sigma$ that satisfy the affine symmetry condition (for a detailed exposition, refer to \cite{affine}).


\begin{theorem}\cite[Corollary 1]{sussmann}\label{coro: minimal_analytic}
Let $f_{\mathcal{N}}$
  be a shallow analytic network represented in the form of equation (\ref{eqn: f_N}) under activation function $\sigma(x)=\mathrm{Sigmoid}(x)$ or $\sigma(x)=\tanh(x)$. Then, $f_{\mathcal{N}}$
  is admissible if and only if it is irreducible.
\end{theorem} 

We now direct our attention to Question \ref{Q3} in the context of shallow analytic networks. While the answer to Question \ref{Q3} is negative for shallow ReLU networks, we demonstrate that an affirmative response can be established for shallow analytic networks.


\begin{theorem}\label{th:gsamp}
Consider any fixed positive integer $m\in \mathbb{Z}_{+}$. There exists a finite point set $\mathcal{X} \subset \mathbb{R}^d$
  such that the following statement holds:
Let $ f_{\mathcal{N}}$  and $f_{\mathcal{N}'}$   be  any two irreducible shallow analytic  networks  represented as 
\begin{equation}\label{eqn: f_n and f_n1}
 f_{\mathcal{N}}(\vx)=\sum_{k=1}^m s_k\cdot \sigma(\langle \va_k,\vx\rangle+b_k)+c\qquad \text{and}\qquad f_{\mathcal{N}'}(\vx)=\sum_{k=1}^m s_k'\cdot \sigma(\langle \va_k',\vx\rangle+b'_k)+c',
\end{equation}
where $\sigma(x)=\mathrm{Sigmoid}(x)$ or $\sigma(x)=\tanh(x)$.
Then, 
if $f_{\mathcal{N}}(\vx)=f_{\mathcal{N'}}(\vx)$ for all $\vx\in \mathcal{X}$, it follows that $\mathcal{N}\sim \mathcal{N}'$.
 \end{theorem}
\begin{remark}
 The proof offers a detailed construction of the finite point set $\mathcal{X}$, with a cardinality of $\#\mathcal{X} =  (\binom{4m}{2}\cdot (d-1)+1) \cdot 2^{2m} $.
   To our knowledge, this theorem is the first to demonstrate exact identification of shallow analytic networks using only a finite number of deterministic samples. Our method differs significantly from previous  studies, such as those discussed in \cite{finite_analytic1,finite_analytic2,minimal_samples}, {in the following three key aspects}:
\begin{enumerate}[{\rm (i)}]
\item It is uniform, providing a fixed set of deterministic sampling points for all networks, unlike previous non-uniform approaches that generate random samples based on specific network.
\item It relies solely on function values at sampling points, without requiring gradient information.
\item It does not impose any restrictions on the number of neurons.
\end{enumerate}

It is worth noting that the substantial size of $\mathcal{X}$ presents a compelling direction for future research: the reduction and optimization of the required number of sampling points. This potential refinement could further enhance the practicality and efficiency of the proposed construction.
\end{remark}
 
\section{Proof of Theorem \ref{th:minm}}\label{sec_min}
We would like to highlight that the following identity plays a crucial role in our argument:
\begin{equation}\label{eq: lemma_temp}
\sigma(x)\equiv x+\sigma(-x),
\end{equation}
where $\sigma(x)=\text{ReLU}(x)$.
We now present a lemma that is useful in proving Theorem \ref{th:minm}.
\begin{lem}\label{independent_relu}
Assume that the hyperplanes $\mathcal{H}(\va_k,b_k)$, $k=1,\ldots,m$, are mutually distinct for some $\va_k\neq \boldsymbol{0}$ and $b_k\in \mathbb{R}$, $k=1,\ldots,m$. Here $\mathcal{H}( \cdot,\cdot )$ is defined in  (\ref{H}). Furthermore, suppose there exist constants $c_0, c_1, \ldots, c_m \in \mathbb{R}$, along with a vector $\boldsymbol{d} \in \mathbb{R}^d$, such that for any $\vx \in \mathbb{R}^d$, the following holds:
\begin{equation}\label{eqn: temp_independent}
\sum_{k=1}^m c_k\cdot \sigma(\langle \va_k,\vx\rangle+b_k)+\langle \boldsymbol{d}, \vx\rangle+c_0=0.
\end{equation}
Then
\[c_0=c_1= \cdots=c_m=0,\qquad \text{and}\qquad \boldsymbol{d}=\boldsymbol{0}. 
\]
\end{lem}
\begin{proof}
Take 
\[
f(\vx)=\sum_{k=1}^m c_k\cdot\sigma(\langle \va_k,\vx\rangle+b_k)+\langle \boldsymbol{d}, \vx\rangle+c_0.
\] 
Denote 
\[
\Omega := \{c_k : c_k \neq 0, \ k=1,\ldots,m\}
\]
 and $S$ as the non-differentiable set of $f$ with $\mathrm{cl}S$ as the closure of the set $S$.
Since $f\equiv 0$, it follows that
 \[
 \cup_{k\in \Omega} \mathcal{H}(\va_k,b_k)={\rm cl} S=\emptyset.
 \] 
  Consequently, we can conclude that $\#\Omega = 0$, which implies that $c_k= 0$, $k=1,\ldots,m$. As a result,  (\ref{eqn: temp_independent}) simplifies to
 \[
 \langle \boldsymbol{d},\vx\rangle+c_0\equiv 0.
 \]
 Thus, we also derive that $c_0 = 0$ and $\boldsymbol{d} = \boldsymbol{0}$. This completes the proof.
\end{proof}

\begin{proof}[Proof of Theorem \ref{th:minm}]
A simple observation is that 
\[
s\cdot\sigma(\langle \va,\vx\rangle+b) \equiv s\cdot \|\va\|_2\cdot \sigma(\langle \va/\|\va\|_2,\vx\rangle+b/\|\va\|_2),
\]
for any $\va\neq \boldsymbol{0}$ and $s,b\in \mathbb{R}$.
Therefore,  Without loss of generality, we can assume that ${\va}_k \in \mathbb{S}^{d-1}$
for all $k \in K_1 \cup K_2$  throughout the remainder of the proof.

We begin by establishing two  results that play a pivotal role in the subsequent proof.
Firstly, for any $\epsilon_k \in \{-1, +1\}$, $k \in K_1$,  the following relationship is established for all $\vx \in \mathbb{R}^{d}$:
\begin{equation}\label{1_temp1}
\begin{aligned}
&\sum_{k\in K_1}(s_{k,1}\cdot \sigma(\langle \va_k,\vx\rangle+b_k)+s_{k,2}\cdot\sigma(\langle -\va_k,\vx\rangle-b_k))\\
=&\sum_{k\in K_1}(s_{k,i_k}\cdot \sigma(\inner{\epsilon_{k}\cdot \va_k,\vx}+\epsilon_k\cdot b_k)+s_{k,3-i_k}\cdot \sigma(\inner{-\epsilon_k\cdot \va_k,\vx}-\epsilon_k\cdot  b_k))\\
=&\sum_{k\in K_1}(s_{k,1}+s_{k,2})\cdot \sigma(\inner{-\epsilon_{k}\cdot \va_k,\vx}-\epsilon_{k}\cdot b_{k})+\sum_{k\in K_1}s_{k,i_k} \cdot (\langle \epsilon_{k}\cdot \va_k,\vx\rangle+\epsilon_k \cdot b_k),
\end{aligned}
\end{equation}
where $i_k$ is taken as in (\ref{i_index}) and  the final equality above is derived from (\ref{eq: lemma_temp}).

Secondly, by iteratively applying (\ref{eq: lemma_temp}), we can derive that, for any subset $K_2'\subseteq K_2$, the following result holds true for all $\vx \in \mathbb{R}^{d}$:
\begin{equation}\label{1_temp2}
\begin{aligned}
&\sum_{k\in K_2}s_k\cdot \sigma (\langle\va_k,\vx\rangle+b_k)=\sum_{k\in K_2'}s_k\cdot \sigma(\inner{\va_k,\vx}+b_k)+\sum_{k \in K_2\setminus K_2'}s_k\cdot \sigma(\inner{\va_k,\vx}+b_k)\\
=&\sum_{k\in K_2'}s_k\cdot \sigma(\inner{-\va_k,\vx}-b_k)+\sum_{k \in K_2\setminus K_2'}s_k\cdot \sigma(\inner{\va_k,\vx}+b_k)+\sum_{k\in K_2'}s_k\cdot (\inner{\va_k,\vx}+b_k).
\end{aligned}
\end{equation}

Having established these results, based on the cardinality of $K_1$, we will now partition the proof into four distinct cases:
(1) $\#K_1 = 0$, (2) $\#K_1 = 1$, (3) $\#K_1 = 2$, (4) $\#K_1 \geq 3$.

\textbf{Case 1: $\#K_1 = 0$}. 
  
In this case, the number of neurons in $f_{\mathcal{N}}$
  is equal to $\#K_2$, i.e., $m=\#K_2$. We will demonstrate that this neuronal count 
 cannot be further reduced.
 Specifically, we will prove that if there exists a network $f_{\mathcal{N}'}: \mathbb{R}^{d}\rightarrow \mathbb{R}$ such that
$f_{\mathcal{N}'}(\vx) = \sum_{k=1}^{m'} s_k'\cdot \sigma(\langle \va_k', \vx \rangle + b_k') + c'$
and $\mathcal{N}\sim \mathcal{N}'$, then it follows that $m' \geq m$.

Let $S$ and $S'$ denote the sets of non-differentiable points of $f_{\mathcal{N}}$ and $f_{\mathcal{N}'}$, respectively. The closures of these sets are represented by $\mathrm{cl} S$ and $\mathrm{cl} S'$.  Consequently, we can establish the following relationship:
\[
\bigcup_{k\in K_2}\mathcal{H}(\va_k,b_k) = \mathrm{cl} S = \mathrm{cl} S' \subseteq \bigcup_{k=1}^{m'} \mathcal{H}(\va_k',b'_k).
\]
Given that the hyperplanes  $\mathcal{H}(\va_k,b_k)$, $k\in K_2$, are mutually distinct,  we can conclude that $m' \geq m=\#K_2$.

\textbf{Case 2: $\#K_1=1$}. 

\textbf{(1) Sufficient part}.
We will demonstrate that if (\ref{cond1: m=1}) holds true, then the number of neurons, $m=2 \#K_1+\#K_2=2+\# K_2$, can indeed be reduced.
By substituting equations (\ref{1_temp1}) and (\ref{1_temp2}) into the formulation of $f_{\mathcal{N}}$, we can derive the following result for all $\vx\in \mathbb{R}^{d}$:
\[
\begin{aligned}
f_{\mathcal{N}}(\vx)=&\sum_{k\in K_1}(s_{k,1}+s_{k,2})\cdot  \sigma(\inner{-\epsilon_{k}\cdot \va_k,\vx}-\epsilon_{k}\cdot  b_{k})+\sum_{k\in K_1}s_{k,i_k}\cdot (\langle \epsilon_{k}\cdot \va_k,\vx\rangle+\epsilon_k  \cdot b_k)\\
&+\sum_{k\in K_2'}s_k\cdot \sigma(\inner{-\va_k,\vx}-b_k)+\sum_{k \in K_2\setminus K_2'}s_k\cdot \sigma(\inner{\va_k,\vx}+b_k)+\sum_{k\in K_2'}s_k\cdot (\inner{\va_k,\vx}+b_k)+c\\
\overset{(a)}=&\sum_{k\in K_1}(s_{k,1}+s_{k,2})\cdot \sigma(\inner{-\epsilon_{k}\cdot \va_k,\vx}-\epsilon_{k}\cdot b_{k})
+\sum_{k\in K_2'}s_k\cdot\sigma(\inner{-\va_k,\vx}-b_k)\\
&+\sum_{k \in K_2\setminus K_2'}s_k\cdot\sigma(\inner{\va_k,\vx}+b_k)
+\sum_{k\in K_1}s_{k,i_k}\cdot \epsilon_k\cdot b_k+\sum_{k\in K_2'}s_k\cdot b_k+c,
\end{aligned}
\]
where  $(a)$ is derived from  (\ref{cond1: m=1}). This indicates a reduction in the number of neurons from $m$ to $\# K_1+\# K_2'+\# (K_2\setminus K_2')=m - 1$.

\textbf{(2) Necessary part}. We will demonstrate that (\ref{cond1: m=1}) necessarily holds when it is possible to reduce the number of neurons in the network $f_{\mathcal{N}}$.

Let us denote the reduced network as $f_{\mathcal{N}'}$, such that $\mathcal{N}\sim \mathcal{N}'$ and  the  neuronal count of $f_{{\mathcal N}'}$ is at most $m-1$. Given that $f_{\mathcal{N}}$
  and $f_{\mathcal{N}'}$ share the same set of non-differentiable points, 
  and considering that the hyperplanes $\mathcal{H}(\va_k, b_k)$, $k \in K_1 \cup K_2$, are mutually distinct with $\#K_1 + \#K_2 = m - 1$, we can conclude that $f_{\mathcal{N}'}$ must comprise exactly $m - 1$ neurons. Furthermore, since the closure of the non-differentiable point set of $f_{\mathcal{N}'}$
  is also given by $\bigcup_{k \in K_1 \cup K_2} \mathcal{H}(\va_k, b_k)$, we can express $f_{\mathcal{N}'}$
 in the following form:
\[
\begin{aligned}
f_{\NN'}(\vx)=&\sum_{k\in K_1}s'_{k}\cdot \sigma(\inner{-\epsilon_{k}\cdot \va_k,\vx}-\epsilon_k\cdot b_k)+\sum_{k\in K_2'}s'_k\cdot \sigma(\inner{-\va_k,\vx}-b_k)\\
&+\sum_{k\in K_2\setminus K_2'}s'_k\cdot\sigma(\inner{\va_k,\vx}+b_k)+c',
\end{aligned}
\] 
for some $\epsilon_k\in\{-1,+1\}$, $k\in K_1$, and $K_2'\subseteq K_2$.
We will now demonstrate that (\ref{cond1: m=1}) is indeed satisfied. Drawing upon (\ref{1_temp1}) and (\ref{1_temp2}), we arrive at the subsequent expression for all $\vx \in \mathbb{R}^{d}$:
\begin{equation}\label{eqn:temp}
\begin{aligned}
0=&f_{\NN}(\vx)-f_{\NN}'(\vx)
=\sum_{k\in K_1}(s_{k,1}+s_{k,2}-s_k')\cdot\sigma(\inner{-\epsilon_{k}\cdot \va_k,\vx}-\epsilon_{k} \cdot  b_{k})\\
&+\sum_{k\in K_2'}(s_k-s_k')\cdot \sigma(\inner{-\va_k,\vx}-b_k)+\sum_{k \in K_2\setminus K_2'}(s_k-s_k')\cdot \sigma(\inner{\va_k,\vx}+b_k)\\
&+\sum_{k\in K_1}s_{k,i_k}\cdot  (\langle \epsilon_{k}\cdot \va_k,\vx\rangle+\epsilon_k \cdot  b_k)+\sum_{k\in K_2'}s_k\cdot (\inner{\va_k,\vx}+b_k)+c-c'.
\end{aligned}
\end{equation}
According to Lemma \ref{independent_relu},  we employ  (\ref{eqn:temp}) to obtain
\[
\sum_{k \in K_1} \epsilon_k \cdot  s_{k,i_k}\cdot   \va_k + \sum_{k \in K_2'} s_k \cdot  \va_k = \boldsymbol{0}.
\]
Consequently, we derive condition (i).

\textbf{Case 3: $\# K_1=2$}.

\textbf{(1) Sufficient part}.
We will demonstrate that if (\ref{cond_m2}) is satisfied, the number of neurons $m$ can be reduced. By incorporating (\ref{1_temp1}) and (\ref{1_temp2}) into the formulation of $f_{\mathcal{N}}$, we can express it as follows for all $\vx \in \mathbb{R}^{d}$:
\[
\begin{aligned}
 f_{\mathcal{N}}(\vx)
=&\sum_{k\in K_1}(s_{k,1}+s_{k,2})\cdot  \sigma(\inner{-\epsilon_{k}\cdot \va_k,\vx}-\epsilon_{k}\cdot   b_{k})+\sum_{k\in K_1}s_{k,i_k}\cdot  (\langle \epsilon_{k}\cdot \va_k,\vx\rangle+\epsilon_k\cdot   b_k)\\
  &+\sum_{k\in K_2'}s_k\cdot \sigma(\inner{-\va_k,\vx}-b_k)+\sum_{k \in K_2\setminus K_2'}s_k \cdot \sigma(\inner{\va_k,\vx}+b_k)+\sum_{k\in K_2'}s_k \cdot (\inner{\va_k,\vx}+b_k)+c\\
&\overset{(b)}=\sum_{k\in K_1}(s_{k,1}+s_{k,2})\cdot  \sigma(\inner{-\epsilon_{k}\cdot \va_k,\vx}-\epsilon_{k}  \cdot b_{k})+\sum_{k\in K_2'}s_k \cdot\sigma(\inner{-\va_k,\vx}-b_k)\\
&+\sum_{k \in K_2\setminus K_2'}s_k\cdot \sigma(\inner{\va_k,\vx}+b_k)-c_0 \cdot \sigma(\langle \va_{k_0},\vx\rangle+b_{k_0})+c_0\cdot  \sigma(-\langle \va_{k_0},\vx\rangle-b_{k_0})\\
&+\sum_{k\in K_1}s_{k,i_k} \cdot \epsilon_k   \cdot b_k+\sum_{k\in K_2'}s_k\cdot  b_k+c_0\cdot  b_{k_0}+c,
 \end{aligned}
\]
where $(b)$  is derived from the following equality:
\[
\begin{aligned}
&\sum_{k\in K_1} s_{k,i_k} \cdot \langle \epsilon_k\cdot \va_k,\vx\rangle+\sum_{k\in K_2'}s_k\cdot  \langle \va_k,\vx\rangle
=-c_0 \cdot\langle \va_{k_0},\vx\rangle-c_0 \cdot  b_{k_0}+c_0 \cdot b_{k_0}\\
&=-c_0 \cdot \sigma(\langle \va_{k_0},\vx\rangle+b_{k_0})+c_0\cdot  \sigma(-\langle \va_{k_0},\vx\rangle-b_{k_0})+c_0 \cdot b_{k_0}.
\end{aligned}
\]
This equality is a direct consequence of  (\ref{cond_m2}) and  (\ref{eq: lemma_temp}). 

Since $k_0 \in K_1 \cup K_2$, one of the functions  $-c_0  \cdot  \sigma(\langle \va_{k_0},  \vx \rangle + b_{k_0})$ and $c_0 \cdot  \sigma(\langle -\va_{k_0},  \vx \rangle - b_{k_0})$ can be integrated into the expression:
\[
\sum_{k\in K_1}(s_{k,1}+s_{k,2})\cdot  \sigma(\inner{-\epsilon_{k}\cdot\va_k, \vx}-\epsilon_{k} \cdot b_{k})+\sum_{k\in K_2'}s_k\cdot \sigma(\inner{-\va_k, \vx}-b_k)+\sum_{k \in K_2\setminus K_2'}s_k\cdot \sigma(\inner{\va_k, \vx }+b_k).
\]
More concretely, the conclusion above holds true under the following three cases:
 \begin{enumerate}[(i)]
 \item  If $k_0\in K_1$, then 
\[
\begin{aligned}
&(s_{k_0,1}+s_{k_0,2})\cdot  \sigma(\inner{-\epsilon_{k_0}\cdot \va_{k_0}, \vx}-\epsilon_{k_0} \cdot b_{k_0})+\epsilon_{k_0} \cdot c_0\cdot  \sigma(\langle -\epsilon_{k_0}\cdot \va_{k_0},\vx \rangle-\epsilon_{k_0}  \cdot b_{k_0})\\
&=(s_{k_0,1}+s_{k_0,2}+\epsilon_{k_0}\cdot  c_0)\cdot  \sigma(\langle -\epsilon_{k_0}\cdot \va_{k_0}, \vx \rangle-\epsilon_{k_0}\cdot  b_{k_0});
\end{aligned}
\]
\item  If $k_0\in K_2'$, then 
\[
s_{k_0}\cdot   \sigma(\inner{-\va_{k_0}, \vx}- b_{k_0})+c_0\cdot   \sigma(\inner{-\va_{k_0}, \vx}- b_{k_0})=(s_{k_0}+c_0)\cdot   \sigma(\inner{-\va_{k_0}, \vx}- b_{k_0});
\]
\item  If $k_0\in K_2\setminus K_2'$, then 
\[
s_{k_0}\cdot   \sigma(\inner{\va_{k_0}, \vx}+b_{k_0})-c_0 \cdot  \sigma(\inner{\va_{k_0}, \vx}+b_{k_0})=(s_{k_0}-c_0)\cdot   \sigma(\inner{\va_{k_0}, \vx}+ b_{k_0}).
\]
\end{enumerate}
Thus, this denotes a reduction in the neuronal number to no more than $\#K_1 + \#K_2 + 1 = m + 1 - \#K_1 = m - 1$.

\textbf{(2) Necessary part}.
When the number of neurons of $f_{\mathcal{N}}$ can be reduced, we will demonstrate that (\ref{cond_m2}) is satisfied. Let us denote the reduced network as $f_{\mathcal{N}'}$. Given that $f_{\mathcal{N}}$ and $f_{\mathcal{N}'}$ share the same set of non-differentiable points, and considering that the hyperplanes  $\mathcal{H}(\va_k, b_k)$, $k \in K_1 \cup K_2$, are mutually distinct with $\#K_1 + \#K_2 = m - 2$, it follows that the number of neurons of $f_{\mathcal{N}'}$ must be either $m - 1$ or $m - 2$. Furthermore, given that the closure of the non-differentiable point set of $f_{\mathcal{N}'}$ is represented by $\bigcup_{k \in K_1 \cup K_2} \mathcal{H}(\va_k, b_k)$, and the number of neurons of $f_{\mathcal{N}'}$ is less than or equal to $m-1$,
{we can articulate $f_{\mathcal{N}'}$ in the following manner for all $\vx \in \mathbb{R}^{d}$:}
\begin{equation}\label{eqn: form1}
\begin{aligned}
f_{\NN'}(\vx)=&\sum_{k\in K_1}s'_{k}\cdot   \sigma(\inner{-\epsilon_{k}\cdot\va_k,\vx}-\epsilon_k\cdot b_k)\\
&+\sum_{k\in K_2'}s'_k \cdot  \sigma(\inner{-\va_k,\vx}-b_k)+\sum_{k\in K_2\setminus K_2'}s'_k \cdot  \sigma(\inner{\va_k,\vx}+b_k)\\
&-c_0 \cdot  \sigma(\langle \va_{k_0},\vx\rangle+b_{k_0})+c_0 \cdot  \sigma(-\langle \va_{k_0},\vx\rangle-b_{k_0})+c',
\end{aligned}
\end{equation}
for some   $\epsilon_k\in\{-1,+1\}$, $k\in K_1$,  $K_2'\subseteq K_2$, $k_0\in K_1\cup K_2$, and $c_0\in \mathbb{R}$, as one of the functions  $-c_0\cdot\sigma(\langle \va_{k_0},  \vx \rangle + b_{k_0})$ and $c_0\cdot\sigma(\langle -\va_{k_0},  \vx \rangle - b_{k_0})$ can be integrated into
\[
\sum_{k\in K_1}s'_k\cdot \sigma(\inner{-\epsilon_{k}\cdot\va_k, \vx}-\epsilon_{k} \cdot  b_{k})+\sum_{k\in K_2'}s'_k \cdot \sigma(\inner{-\va_k, \vx}-b_k)+\sum_{k \in K_2\setminus K_2'}s'_k\cdot  \sigma(\inner{\va_k, \vx}+b_k),
\]
 following similar lines of argument as previously presented.
Subsequently, based on (\ref{1_temp1}) and (\ref{1_temp2}), we obtain the following for all  $\vx \in \mathbb{R}^{d}$:
\begin{equation}\label{eqn: temp_k2}
\begin{aligned}
0 =&f_{\NN}(\vx)-f_{\NN}'(\vx)\\
=&\sum_{k\in K_1}(s_{k,1}+s_{k,2}-s_k') \cdot  \sigma(\inner{-\epsilon_{k}\cdot\va_k,\vx}-\epsilon_{k}\cdot b_{k})+\sum_{k\in K_1}s_{k,i_k}\cdot   (\langle \epsilon_{k}\cdot\va_k,\vx\rangle+\epsilon_k \cdot b_k)\\
&+\sum_{k\in K_2'}(s_k-s_k')\cdot  \sigma(\inner{-\va_k,\vx}-b_k)+\sum_{k \in K_2\setminus K_2'}(s_k-s_k')\cdot  \sigma(\inner{\va_k,\vx}+b_k)\\
&+\sum_{k\in K_2'}s_k\cdot  (\inner{\va_k,\vx}+b_k)+c_0 \cdot  \sigma(\langle \va_{k_0},\vx\rangle+b_{k_0})-c_0\cdot   \sigma(-\langle \va_{k_0},\vx\rangle-b_{k_0})+c-c'.
\end{aligned}
\end{equation}
Since  (\ref{eq: lemma_temp}) leads to the identity
\[
c_0\cdot    \sigma(\langle \va_{k_0}, \vx \rangle + b_{k_0}) - c_0 \cdot   \sigma(-\langle \va_{k_0}, \vx \rangle - b_{k_0}) \equiv c_0\cdot (\langle \va_{k_0}, \vx \rangle + b_{k_0}),
\]
and by invoking Lemma \ref{independent_relu},  (\ref{eqn: temp_k2}) arrives at  the following relationship:
\[
\sum_{k \in K_1} \epsilon_k \cdot   s_{k,i_k}\cdot    \va_k + \sum_{k \in K_2'} s_k \cdot   \va_k + c_0 \cdot   \va_{k_0} = \boldsymbol{0}.
\]
Thus, we can conclude that condition (ii) holds.

\textbf{Case 4: $\#K_1\geq 3$}.
Building upon  (\ref{1_temp1}) and (\ref{eq: lemma_temp}), we can directly conclude that,  for all $\epsilon_k \in \{-1, +1\}$, $k \in K_1$, it obtains that:
\begin{equation*}
\begin{aligned}
f_{\mathcal{N}}(\vx)=&\sum_{k\in K_1}(s_{k,1}+s_{k,2}) \cdot  \sigma(\inner{-\epsilon_{k}\cdot\va_k,\vx}-\epsilon_{k}\cdot b_{k})+\sum_{k\in K_1}s_{k,i_k}   (\langle \epsilon_{k}\cdot\va_k,\vx\rangle+\epsilon_k \cdot b_k)\\
&+\sum_{k\in K_2}s_k\cdot  \sigma (\langle\va_k,\vx\rangle+b_k)+c\\
=&\sum_{k\in K_1}(s_{k,1}+s_{k,2}) \cdot  \sigma(\inner{-\epsilon_{k}\cdot\va_k,\vx}-\epsilon_{k}\cdot b_{k})+\sum_{k\in K_2}s_k\cdot  \sigma (\langle\va_k,\vx\rangle+b_k)\\
&+\sigma\Big(\big\langle \sum_{k\in K_1}s_{k,i_k} \cdot  \epsilon_k \cdot  \va_k,\vx\big\rangle\Big)-\sigma\Big(\big\langle -\sum_{k\in K_1}s_{k,i_k}\cdot   \epsilon_k  \cdot \va_k,\vx\big\rangle\Big)+\sum_{k\in K_1}s_{k,i_k} \cdot  \epsilon_k\cdot   b_k+c.
\end{aligned}
\end{equation*}
Consequently, the number of neurons can be reduced to at most
$\#K_1 + \#K_2 + 2 \leq m - 1$, given that $2 \#K_1 + \#K_2 = m$ and $\#K_1 \geq 3$.
\end{proof}

\section{Proof of Theorem \ref{th:impos}}\label{sec: negative_answer}

\begin{proof}[Proof of Theorem \ref{th:impos}]
Consider $f_{\mathcal{N}}$ and $f_{\mathcal{N'}}$ defined as 
\[
\begin{aligned}
f_{\mathcal{N}}(\vx):=&\sigma(\langle \vw+\epsilon\cdot \vn,\vx\rangle+b)+\sigma(\langle \vw-\epsilon \cdot \vn,\vx\rangle+b)+\sum_{k=3}^{m}\sigma(\langle \va_k,\vx\rangle+b_k),\\
f_{\mathcal{N}'}(\vx):=&\sigma(\langle \vw+\epsilon' \cdot \vn,\vx\rangle+b)+\sigma(\langle \vw-\epsilon' \cdot \vn,\vx\rangle+b)+\sum_{k=3}^{m}\sigma(\langle \va_k,\vx\rangle+b_k),
\end{aligned}
\]
where $b \in \mathbb{R}$, $\vw \in \mathbb{R}^{d}$, $\vn \in \mathbb{S}^{d-1}$, $\epsilon'> \epsilon>0$, $\{\va_k\}_{k=3}^m\subseteq \mathbb{R}^{d}$ and $\{b_k\}_{k=3}^{m}\subset \R$. Here these parameters satisfy the following conditions:
\begin{enumerate}[(i)]
 \item  For all $j \in\{ 1, \ldots, n\}$, $\langle \vw, \vx_j \rangle + b \neq 0$; 
 \item $\langle  \vn,\vw \rangle = 0$;
 \item For all $j \in\{ 1, \ldots, n\}$, $\epsilon\cdot  |\langle \vn, \vx_j \rangle| < |\langle \vw, \vx_j \rangle + b|$ and $\epsilon' \cdot |\langle \vn, \vx_j \rangle| < |\langle \vw, \vx_j \rangle + b|$;
 \item The hyperplanes
 \[
\mathcal{H}(\vw+\epsilon\cdot \vn,b),\ \mathcal{H}(\vw-\epsilon\cdot \vn,b),\ \mathcal{H}(\vw+\epsilon'\cdot \vn,b),\ \mathcal{H}(\vw+\epsilon'\cdot \vn,b),\  \mathcal{H}(\va_k,b_k), \ \text{for}\ k=3,\ldots,m,
 \]
 are mutually distinct.
\end{enumerate}
According to Theorem \ref{th:minm},  $f_{\mathcal{N}}$ and $f_{\mathcal{N}'}$ are irreducible. Through straightforward computation, we derive that
\[
f_{\mathcal{N}}(\vx_{j})=f_{\mathcal{N}'}(\vx_{j})=\begin{cases}
2(\langle \vw,\vx_j\rangle+b)+\sum_{k=3}^{m}\sigma(\langle \va_k,\vx_j\rangle+b_k), & \ \text{if}\ \langle \vw,\vx_j\rangle+b>0;\\
\sum_{k=3}^{m}\sigma(\langle \va_k,\vx_j\rangle+b_k), &\  \text{otherwise}.
\end{cases}
\]
Moreover, given that $\langle\vw, \vn\rangle =0$, we can deduce:
\[
\{\vx\in \mathbb{R}^d : \langle \vw, \vx \rangle + b = 0\} \cap \{\vx\in \mathbb{R}^d : \langle \vn, \vx \rangle > 0\} \neq \emptyset.
\]
This implies the existence of $\vx_0 \in \mathbb{R}^{d}$ such that $\langle \vw, \vx_0 \rangle + b = 0$ and $\langle \vn, \vx_0 \rangle > 0$. Consequently, we have
\[f_{\mathcal{N}}(\vx_0) = \epsilon\cdot \langle \vn,\vx_0\rangle +\sum_{k=3}^{m}\sigma(\langle \va_k,\vx_0\rangle+b_k)\neq \epsilon' \cdot \langle \vn,\vx_0\rangle+\sum_{k=3}^{m}\sigma(\langle \va_k,\vx_0\rangle+b_k)= f_{\mathcal{N}'}(\vx_0),\]
which implies that $\mathcal{N} \not\sim \mathcal{N}'$.
\end{proof}

\section{Proof of Theorem \ref{th:samplingc}}\label{sec: finite_sampling_relu}\label{sec: prof_finite_relu}

 To lay the groundwork for our analysis, we first introduce a  relationship between the parameters of equivalent shallow ReLU networks. This relationship, reformulated in Theorem \ref{prop:ambiguity} below, is derived from \cite[Theorem 1]{notes} and proves instrumental in our subsequent examination of finite samplings for $f_{\mathcal N}$. While the original theorem in \cite{notes} presents only necessary conditions, we note that the sufficiency of these conditions can be readily established. For brevity, we omit the detailed proofs here.

 \begin{theorem}\cite[Theorem 1]{notes}\label{prop:ambiguity} 
 Let $f_{\mathcal{N}}$   and $f_{\mathcal{N'}}$ be irreducible shallow ReLU networks of the form:
\[
f_{\mathcal{N}}(\vx): = \sum_{k=1}^{m} s_{k} \cdot  \sigma(\langle \va_{k}, \vx \rangle + b_k) + c\quad 
\text{and}\quad f_{\mathcal{N'}}(\vx) := \sum_{k=1}^{m'} s'_{k} \cdot  \sigma(\langle \va'_{k}, \vx \rangle + b_k') + c',
\]
such that the hyperplanes $\mathcal{H}(\va_{k}, b_k),\ k=1,\ldots,m$ $($resp.  $\mathcal{H}(\va'_{k}, b'_k), \ k=1,\ldots,m'$$)$ are mutually distinct. Then, $\mathcal{N} \sim \mathcal{N}'$ if and only if the following conditions are satisfied:
\begin{enumerate}[{\rm (i)}]
\item $m=m'$;
\item 
 There exists a permutation $\pi$ of $\{1,2,\ldots,m\}$, and $\epsilon_k \in \{-1,+1\}$ and $\lambda_k > 0$, $k=1,\ldots,m$,
  such that: 
\begin{equation}\label{eq:eak1}
\epsilon_k\cdot    \lambda_k\cdot    {(\va_{k},b_k)}={(\va'_{\pi(k)},b'_{\pi(k)})}\quad \text{and}\quad
\frac{1}{\lambda_k}\cdot    s_{k} =s'_{\pi(k)},\quad k=1,\ldots,m;
\end{equation}
\item 
Setting $K:=\{k :\epsilon_k=-1, k=1,\ldots,m\}$, we have
\begin{equation}\label{equiv_extra}
\sum_{k\in K}s_k\cdot \va_k=\boldsymbol{0}\qquad \text{and}\qquad c'=\sum_{k\in K}s_k \cdot b_k+c.
\end{equation}
If $K = \emptyset$, (\ref{equiv_extra}) reduces to $c = c'$.
\end{enumerate}
\end{theorem}

To pave the way for our subsequent proof and enhance clarity, we introduce the following key concept, which proves instrumental in constructing the set of sampling points.

\begin{defi}
\label{def: feasible}
Consider $f_{\mathcal{N}}$   in the form of equation (\ref{eqn: f_N}) with the family of hyperplanes $\{\mathcal{H}(\va_k,b_k)\}_{k=1}^m$ for  $\va_k\in \mathbb{R}^{d}$ and $b_k\in \mathbb{R}$, $k=1,\ldots,m$. Take $\vu_{j},\vv_{j}\in \mathbb{R}^{d}$,  $j=1,\ldots,md$.
For each $j\in\{1,\ldots,md\}$, denote $\mathcal{L}_j:=\{\vu_{j}+t\cdot\vv_{j} : t\in \mathbb{R}\}$.
We say that the collection of the lines $\{\mathcal{L}_j\}_{j=1}^{md}$
is feasible with respect to $f_{\mathcal{N}}$, if:

\begin{enumerate}[{\rm (i)}]
\item ${\rm rank}([\boldsymbol{v}_1, \ldots, \boldsymbol{v}_{md}])=d$;
\item Denote the  intersection points between $\{\mathcal{L}_{j}\}_{j=1}^{md}$ and  $\{\mathcal{H}(\va_k, b_k)\}_{k=1}^m$ as 
\begin{equation}\label{eqn: z_k}
\vz_{j,k}:=\mathcal{L}_j\cap  \mathcal{H}(\va_k, b_k), \quad 1\leq j\leq md\quad \text{and} \quad 1\leq k\leq m.
\end{equation}
The points $\vz_{j,k}$, $j=1,\ldots,md$, and $k=1,\ldots,m$, are mutually distinct;
 \item For any point set 
 \[
 {\widetilde{\mathcal{Z}}}:=\{\widetilde{\vz}_1,\ldots,\widetilde{\vz}_d\}\subseteq \{{\vz}_{j,k}\}_{1\leq j\leq md,1\leq k\leq m},
 \]
   if $\widetilde{\mathcal{Z}}\subseteq \mathcal{H}(\va_{k_0},b_{k_0})$ for some $k_0\in \{1,\ldots,m\}$, then $\widetilde{\mathcal{Z}}$ can uniquely determine $\mathcal{H}(\va_{k_0},b_{k_0})$, or equivalently,
\begin{equation}\label{eqn: span_property}
\rank([\widetilde{\vz}_2-\widetilde{\vz}_1,\widetilde{\vz}_3-\widetilde{\vz}_1,\ldots,\widetilde{\vz}_{d}-\widetilde{\vz}_1])=d-1.
\end{equation}
\end{enumerate}
\end{defi}
\begin{remark}
\label{rem: key_rem_correct}
For almost all sets of vectors $\{\boldsymbol{v}_j\}_{j=1}^{d} \subseteq \mathbb{R}^d$, condition $\mathrm{(i)}$ holds. Assume that the hyperplanes $\mathcal{H}(\va_k, b_k)$, $k = 1, \ldots, m$, are mutually distinct. Then, for almost all lines $\mathcal{L}$, the intersection points between $\mathcal{L}$ and the hyperplanes $\mathcal{H}(\va_k, b_k)$, $k=1,\ldots,m$, are distinct. Furthermore, for almost all sets of $d$ points in a given hyperplane $\mathcal{H}(\va_{k_0}, b_{k_0})$  for some $k_0 \in \{1, \ldots, m\}$, the condition expressed in (\ref{eqn: span_property}) is satisfied. Consequently, for almost all parameters $\vu_j, \vv_j \in \mathbb{R}^d$, $j = 1, \ldots, md$, the collection $\{\mathcal{L}_j\}_{j=1}^{md}$ is feasible with respect to $f_{\mathcal{N}}$.
\end{remark}

The following lemma demonstrates that if the collection of lines $\{\mathcal{L}_k\}_{k=1}^{md}$
 is feasible with respect to $f_{\mathcal{N}}$, then for any other neural network $f_{\mathcal{N'}}$, the condition $\mathcal{N} \sim \mathcal{N'}$
  holds, provided that the function values of these two networks are identical on the lines $\{\mathcal{L}_k\}_{k=1}^{md}$. This lemma facilitates the proof of Theorem \ref{th:samplingc}.
\begin{lem}\label{lem: tec_hyperplane}
Let $f_{\mathcal{N}}: \mathbb{R}^d \to \mathbb{R}$ be an irreducible shallow ReLU network with $m$ neurons, expressed in the form of (\ref{eqn: f_N}). Assume that the hyperplanes $\mathcal{H}(\boldsymbol{a}_k, b_k)$, $k=1,\ldots,m$, are mutually distinct. Take the collection of lines $\{\mathcal{L}_j\}_{j=1}^{md}$ that is feasible with respect to $f_{\mathcal N}$. For any irreducible shallow ReLU network $f_{\mathcal N'}$ with $m$ neurons, if
$f_{\mathcal N}(\vx)=f_{\mathcal N'}(\vx)$ for all $\vx\in \cup_{j=1}^{md}\mathcal{L}_j$, then
${\mathcal N}\sim {\mathcal N}'$.
\end{lem}

\begin{proof}
Similarly to the proof of Theorem \ref{th:minm},  since 
\[
s\cdot \sigma(\langle \va,\vx\rangle+b) \equiv s\cdot \|\va\|_2\cdot \sigma(\langle \va/\|\va\|_2,\vx\rangle+b/\|\va\|_2),
\]
for any $\va\neq \boldsymbol{0}$ and $s,b\in \mathbb{R}$, we may assume that $\va_k,\va'_k \in \mathbb{S}^{d-1}$, for $k = 1, \ldots, m$.   The subsequent proof is divided into two steps:
 
   \textbf{Step 1: Establish  $\{\mathcal{H}(\va_{k}, b_{k})\}_{k=1}^{m} = \{\mathcal{H}(\va'_{k}, b'_{k})\}_{k=1}^{m}$. } 
   
   For each $j\in\{1,\ldots,md\}$, denote $\mathcal{L}_j:=\{\vu_j+t\cdot \vv_{j} :\ t\in \mathbb{R}\}$ for some $\vu_j,\vv_j\in \mathbb{R}^{d}$ and $f_{{\mathcal N},j}: \mathbb{R}\rightarrow \mathbb{R}$ with $f_{{\mathcal N},j}(t):=f_{\mathcal{N}}(\vu_{j}+t\cdot \vv_{j})$. Since $\{\mathcal{L}_j\}_{j=1}^{md}$
 is feasible with respect to $f_{\mathcal{N}}$  and  the number of neurons of $f_{\mathcal{N}}$ is $m$, based on condition (ii) in Definition \ref{def: feasible}, each $f_{{\mathcal N},j}$ is partitioned into $m+1$ intervals,  where each interval contains a linear function.  Therefore,  for each $j=1,\ldots,md$, we can utilize $f_{\mathcal{N},j}$  to identify the partition points $\vz_{j,k}$, for $k=1,\ldots,m$, where $\vz_{j,k}$ is defined in (\ref{eqn: z_k}).

We will demonstrate that the set of hyperplanes $\{\mathcal{H}(\va_{k}, b_{k})\}_{k=1}^{m}$
can be uniquely determined by identifying those hyperplanes that contain $md$ elements from the set 
\[
\{\vz_{j,k}\}_{1 \leq j \leq md, 1 \leq k \leq m}.
\]
 To this end, we propose the following claim:
For any subset $\mathcal{Z}_0 \subseteq \{{\vz}_{j,k}\}_{1 \leq j \leq md, 1 \leq k\leq m}$  with cardinality $\#\mathcal{Z}_0 = md$, one of the following two conditions must hold: 
either 
\begin{equation}
\label{eqn: necessary_1}
\mathcal{Z}_0\subseteq \mathcal{H}(\va_{k_0},b_{k_0})\quad  \text{for some}\  k_0\in \{1,\ldots,m\}
\end{equation}
 or 
 \begin{equation}
 \label{eqn: necessary_2}
 \mathcal{Z}_0\not\subseteq \mathcal{H}(\va,b)\quad \text{for any} \ \va\in \mathbb{R}^{d}\setminus \{\boldsymbol 0\}\ \text{and}\ b\in \mathbb{R}.
 \end{equation} 
Indeed, for each $k=1,\ldots,m$, the hyperplane $\mathcal{H}(\va_{k}, b_{k})$ contains precisely $md$ points from the set $\{\vz_{j,k}\}_{1 \leq j \leq md, 1 \leq k \leq m}$. As a result, the collection of hyperplanes $\{\mathcal{H}(\va_{k}, b_{k})\}_{k=1}^{m}$ is uniquely determined by this point set.

 To demonstrate that either (\ref{eqn: necessary_1}) or (\ref{eqn: necessary_2}) must hold, we proceed as follows:
For each $j \in\{1, \ldots, md\}$, according to condition (ii) in Definition \ref{def: feasible}, there exists a one-to-one correspondence between $\{{\vz}_{j,k}\}_{k=1}^m$ and $\{\mathcal{H}(\va_k, b_k)\}_{k=1}^m$. Since $\#\mathcal{Z}_0 = md$, based on {the pigeonhole principle}, there exists an index $k_0 \in \{1, \ldots, m\}$ such that $\mathcal{H}(\va_{k_0}, b_{k_0})$ contains at least $d$ elements in $\mathcal{Z}_0$.
If $\mathcal{Z}_0 \subseteq \mathcal{H}(\va_{k_0}, b_{k_0})$, then the result in (\ref{eqn: necessary_1}) is satisfied. 
Otherwise, we can take $\mathcal{Z}_1 \subseteq \mathcal{Z}_0 \cap \mathcal{H}(\va_{k_0}, b_{k_0})$ 
and $\mathcal{Z}_2 \subseteq \mathcal{Z}_0 \setminus \mathcal{H}(\va_{k_0}, b_{k_0})$ such that  $\#\mathcal{Z}_1 = d$  and  $\#\mathcal{Z}_2 = 1$.
According to condition (iii) in Definition \ref{def: feasible}, if there exists a hyperplane $\mathcal{H}(\va,b) \subset \mathbb{R}^d$ such that $\mathcal{Z}_1 \subseteq \mathcal{H}(\va,b)$, then $\mathcal{H}(\va,b)$  must coincide with $\mathcal{H}(\va_{k_0}, b_{k_0})$. Consequently, given the definition of $\mathcal{Z}_2$,  (\ref{eqn: necessary_2}) is established.

To determine the family of hyperplanes $\{\mathcal{H}(\va'_{k}, b'_{k})\}_{k=1}^{m}$, we can employ an approach analogous to the one used earlier: denote the  intersection points between the lines $\{\mathcal{L}_{j}\}_{j=1}^{md}$ and the hyperplanes $\{\mathcal{H}(\va'_k, b'_k)\}_{k=1}^m$ as  $\{\vz'_{j,k}\}_{1\leq j\leq md, 1\leq k\leq m}$.
Given that $f_{\mathcal{N}}(\vx) = f_{\mathcal{N}'}(\vx)$ for all $\vx\in \cup_{j=1}^{md}\mathcal{L}_j$, and based on condition  {\rm (ii)} in Definition \ref{def: feasible} in conjunction with the presence of $m$ neurons in $f_{\mathcal{N}'}$, we can readily conclude that
\[
\{\vz_{j,k}\}_{1\leq j\leq md, 1\leq k\leq m}=\{\vz'_{j,k}\}_{1\leq j\leq md, 1\leq k\leq m}.
\]

Following a similar line of reasoning as in the preceding analysis (which is not explicitly shown here),
we can uniquely determine $\{\mathcal{H}(\va'_{k}, b'_{k})\}_{k=1}^{m}$ based on the set $\{{\vz}'_{j,k}\}_{1 \leq j \leq md, 1 \leq k \leq m}$. Consequently, since $\{\vz_{j,k}\}_{1\leq j\leq md, 1\leq k\leq m}=\{\vz'_{j,k}\}_{1\leq j\leq md, 1\leq k\leq m}$, it follows that 
\[
\{\mathcal{H}(\va_{k}, b_{k})\}_{k=1}^{m} = \{\mathcal{H}(\va'_{k}, b'_{k})\}_{k=1}^{m}.
\]

\textbf{Step 2: Establishing the Equivalence Relation $\mathcal{N} \sim \mathcal{N}'$.}

Building upon the conclusion in \textbf{Step 1}, and noting that $\va_k, \va'_k \in \mathbb{S}^{d-1}$, $k = 1, \ldots, m$, the pairs $(\va_k,b_k),(\va'_k,b'_k)\in \mathbb{R}^{d}\times \mathbb{R}$, $k=1,\ldots,m$, satisfy $\{\epsilon_k\cdot (\va_k,b_k)\}_{k=1}^{m}=\{({\va'_k},b'_k)\}_{k=1}^{m}$ for some  $\epsilon_k\in\{-1,1\}$, $k=1,\ldots,m$.  Without loss of generality, take $\va'_k=\epsilon_k  \va_k$ and $b'_k=\epsilon_k  b_k$. Then the network $f_{\mathcal{N}'}$ can be expressed as:
\begin{equation}\label{eqn: f_new}
\begin{aligned}
f_{\mathcal{N}'}(\vx):=&\sum_{k=1}^m s'_k \cdot \sigma(\langle \epsilon_k \cdot  \va_k,\vx\rangle+\epsilon_k  \cdot b_k)+c'\\
=&\sum_{k\notin K}s'_k \cdot \sigma(\langle \va_k,\vx\rangle+b_k)+\sum_{k\in K}s'_k\cdot  \sigma(-\langle \va_k,\vx\rangle-b_k)+c'\\
=&\sum_{k=1}^{m} s'_k\cdot  \sigma(\langle \va_k,\vx\rangle+b_k)-\sum_{k\in K}s'_k\cdot  (\langle \va_k,\vx\rangle+b_k)+c'\\
=&\sum_{k=1}^{m} s'_k\cdot  \sigma(\langle \va_k,\vx\rangle+b_k)+\langle \vp,\vx \rangle+q+c',
\end{aligned}
\end{equation}
where $K:=\{k:\epsilon_k=-1, 1\leq k\leq m\}$,  $\vp:=-\sum_{k\in K}s'_k \cdot \va_k$ and $q:=-\sum_{k\in K}s'_k \cdot  b_k$.
  Here, the last equality above follows from the identity $\sigma(-x)\equiv\sigma(x)-x$. 
For each $j\in \{1,\ldots,md\}$, since $f_{\mathcal{N}}(\vx)=f_{\mathcal{N}'}(\vx)$ for all $\vx\in \mathcal{L}_j$ we have
 \begin{equation}\label{eq:Ldeng}
f_{\mathcal{N}}(\vu_{j}+t\cdot \vv_{j}) = f_{\mathcal{N}'}(\vu_{j}+t\cdot \vv_{j})\quad \text{for}\ \text{all}\ t\in \mathbb{R},
\end{equation}
 which implies
 \begin{equation}\label{eqn: lemma_apply}
 \sum_{k=1}^{m}(s'_k-s_k)\cdot \sigma(\alpha_{j,k}  \cdot t+\beta_{j,k})+d_{j} \cdot  t+c_j=0,\quad \text{for}\ \text{all}\ t\in \mathbb{R}.
 \end{equation}
 Here, for $j=1,\ldots,md$ and $k=1,\ldots,m$, 
\[
\alpha_{j,k}:=\langle \va_k,\vv_j\rangle,\  \beta_{j,k}:=\langle \va_k,\vu_j\rangle+b_k, \ d_j:=\langle \vp,\vv_j\rangle\  \text{and} \ c_j:=\langle \vp,\vu_j\rangle+q+c'-c.
\]
 
From condition (ii) in Definition \ref{def: feasible}, the hyperplanes (collapsing to points) 
\[
\mathcal{H}(\alpha_{j,1},\beta_{j,1}),\ldots,\mathcal{H}(\alpha_{j,m},\beta_{j,m})
\]
 are distinct. Applying Lemma \ref{independent_relu} to (\ref{eqn: lemma_apply}), we immediately obtain 
 \[
 s'_k=s_k,\quad k=1,\ldots,m,\quad \text{and}\quad d_j=c_j=0,\quad j=1,\ldots,md.
 \]

  Given that 
  \[
  \langle \vp,\vv_j\rangle=d_j=0, \quad j=1,\ldots, md,
  \]
 we can conclude that $\vp=\boldsymbol{0}$. 
 Here, we use condition (i) in Definition \ref{def: feasible}, which states that ${\rm rank}([\boldsymbol{v}_1, \ldots, \boldsymbol{v}_{md}])=d$.
 Hence, we have 
 \begin{equation}\label{eqn: temp_main1}
 -\vp=\sum_{k\in K}s'_k\cdot  \va_k= \sum_{k\in K}s_k\cdot  \va_k=\boldsymbol{0}. 
 \end{equation}
 Meanwhile, since 
 \[
 0=c_j=\langle \vp,\vu_j\rangle+q+c'-c=q+c'-c,\quad j=1,\ldots,md,
 \]
   we obtain 
 \begin{equation}\label{eqn: temp_main2}
c'=c-q=c+\sum_{k\in K}s'_k \cdot  b_k=c+\sum_{k\in K}s_k  \cdot b_k,
 \end{equation}
with $s'_k=s_k$, $\va'_k=\epsilon_k\cdot  \va_k$, and $b'_k=\epsilon_k \cdot b_k$
for $k=1,\ldots,m$, and using equations (\ref{eqn: temp_main1}) and (\ref{eqn: temp_main2}), we can apply Theorem \ref{prop:ambiguity} to conclude that $\mathcal{N} \sim \mathcal{N}'$. This completes the proof.
\end{proof}

 \begin{figure}[!t]
	\centering
	\includegraphics[scale=0.24]{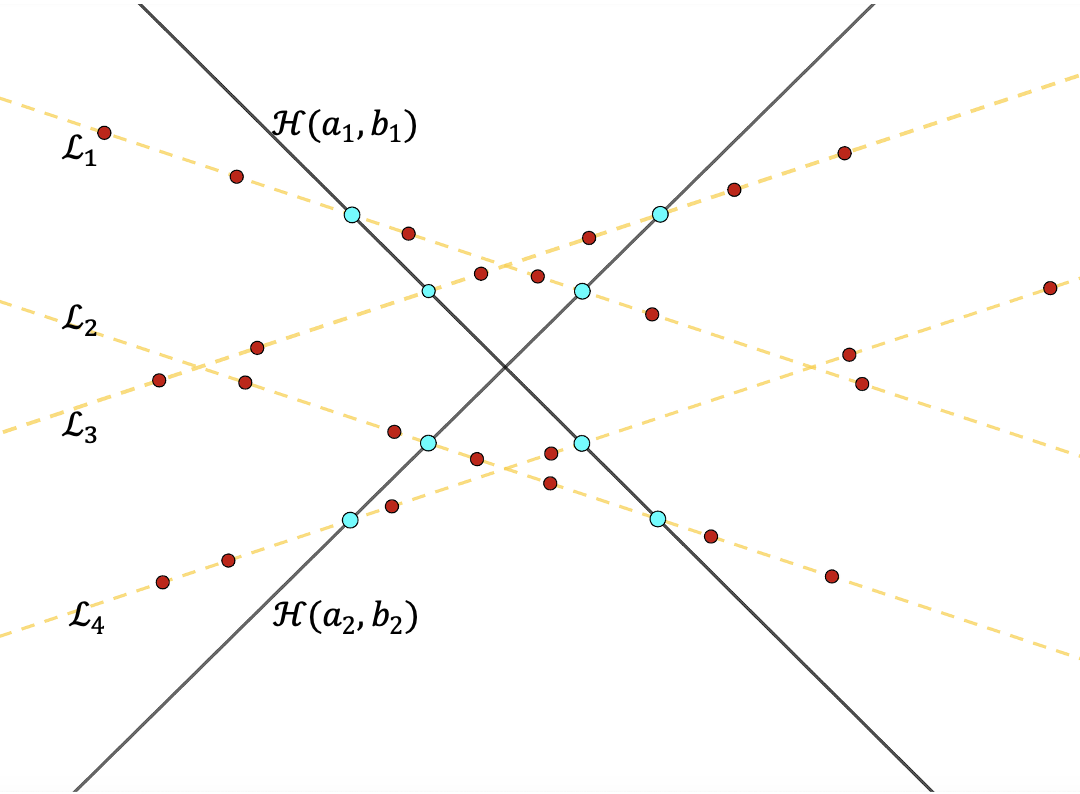}
	\caption{\small
Uniquely determining parameters ${\mathcal N}$ using $f_{\mathcal N}$
  values at specific sampling points.
 Here $f_{\mathcal N}(\vx)=\sigma(\innerp{\va_1,\vx}+b_1)+\sigma(\innerp{\va_2,\vx}+b_2)$
where $\va_1=(1,1)^\top, b_1=0$ , $\va_2=(1,-1)^\top, b_2=0$, $m=2$, and $d=2$.
Red points: Sampling points from the proof of Theorem \ref{th:samplingc}, used to determine $f_{\mathcal N}$
  values along lines ${\mathcal L}_j, j=1,\ldots,4$.
Blue points: Intersections of ${\mathcal L}_j$  with ${\mathcal H}(\va_1,b_1)$
  and ${\mathcal H}(\va_2,b_2)$, derived from $f_{\mathcal N}$
  values on ${\mathcal L}_j, j=1,\ldots,4$. These intersections uniquely determine ${\mathcal H}(\va_1, b_1)$, ${\mathcal H}(\va_2, b_2)$, and consequently, ${\mathcal N}$.}
	\label{fig:ps}
\end{figure}

We next present the proof of Theorem \ref{th:samplingc}.
To elucidate the steps of this proof, Figure \ref{fig:ps} presents a concrete example in $\R^2$, visualizing the sampling points and their role in parameter determination.

\begin{proof}[Proof of Theorem \ref{th:samplingc}]
Here we  present a systematic procedure for constructing the point set $\mathcal{X}$:
\begin{enumerate}[(i)]
\item
 Take  the collection of lines $\{\mathcal{L}_j\}_{j=1}^{md}$, which is feasible with respect to $f_{\mathcal{N}}$ (see Definition \ref{def: feasible}).  Assume that  $\mathcal{L}_j:=\{\vu_{j}+t \cdot \vv_{j}\ :\ t\in \mathbb{R}\}$ for some $\vu_j,\vv_j\in \mathbb{R}^{d}$, $j=1,\ldots,md$.    Without loss of generality, for each $j\in\{1,\ldots,md\}$, take a permutation $\pi_j$ of $\{1,\ldots,m\}$ such that 
{ \[
 {\vz}_{j,\pi_j(k)}:=\vu_{j}+w_{j,k}\cdot\vv_{j}\in \mathcal{L}_{j}\cap {\mathcal{H}}(\va_{\pi_j(k)},b_{\pi_j(k)}),\quad k=1,\ldots,m,
 \]  }
 where $w_{j,k}$,  $k=1,\ldots,m$, are ordered as:
 \begin{equation}\label{eqn: w}
w_{j,0}:=-\infty< {w}_{j,1}< \cdots<{w}_{j,m}<w_{j,m+1}:=+\infty.
 \end{equation}
\item Define the point set ${\mathcal{X}}$ as follows:
\[
{\mathcal{X}}:= \{\vx_{j,l}\}_{1 \leq j \leq md, 1 \leq l \leq 2m+2}\subseteq \R^d,
\]
where the points $\vx_{j,l}$  satisfy the following conditions:
 \begin{enumerate}[(a)]
 \item for each $j \in \{1,\ldots,md\}$, let ${\vx}_{j,l}:=\vu_{j}+{t}_{j,l}\cdot \vv_{j}\in \mathcal{L}_j$, $l=1,\ldots,2m+2$, such that $t_{j,2k-1}$ and $t_{j,2k}$ are distinct and $t_{j,2k-1},t_{j,2k}\in (w_{j,k-1},w_{j,k})$,  $k=1,\ldots,m+1$;
 
 \item  for any three distinct points in the set $\mathcal{X}$, if they are collinear, let $\mathcal{L}_0$
  denote the line containing these points. Then, $\mathcal{L}_0$
  must coincide with one of the predefined lines $\mathcal{L}_j$, where $j \in \{1, \ldots, md\}$.

 \end{enumerate}
\end{enumerate}
 Drawing from Remark \ref{rem: key_rem_correct} and the definition of $\mathcal{X}$, we can assert that $\mathcal{X}$ can be successfully constructed for almost all choices of $\{\vu_j\}_{1\leq j\leq md}, \{\vv_j\}_{1\leq j\leq md} \subseteq \mathbb{R}^{d}$
and $t_{j,2k-1},t_{j,2k}\in (w_{j,k-1},w_{j,k})$, for $j=1,\ldots,md$ and $k=1,\ldots,m+1$,
where $w_{j,k}$  is defined as in  (\ref{eqn: w}).

Based on the construction procedure (ii) and condition (b) outlined above, the only allowable set of $2m + 2$ collinear points in $\mathcal{X}$ must be $\{\vx_{j,l}\}_{l=1}^{2m+2}$, $j=1, \ldots, md$, which uniquely determines the set of lines $\{\mathcal{L}_j\}_{j=1}^{md}$ from the point set $\mathcal{X}$.
 For $j\in \{1,\ldots,md\}$, set $f_{{\mathcal N},j}: \mathbb{R}\rightarrow \mathbb{R}$ with $f_{{\mathcal N},j}(t):=f_{\mathcal{N}}(\vu_{j}+t\cdot \vv_{j})$.
 Since $\{\mathcal{L}_j\}_{j=1}^{md}$ is feasible with respect to $f_{\mathcal{N}}$, and the number of neurons in $f_{\mathcal{N}}$ is $m$,  $f_{\mathcal{N},j}$  can be divided into $m + 1$ intervals, each containing a linear function.  Applying condition (ii) in Definition \ref{def: feasible} along with construction procedure (ii) and condition (a) above, we conclude that $\{(\vx_{j,l}, f_{\mathcal N}(\vx_{j,l}))\}_{l=1}^{2m+2}$ uniquely determines all the function values of $f_{\mathcal{N}}$ along the line $\mathcal{L}_j$, for each $j \in\{1, \ldots, md\}$. Therefore, for any other network $f_{\mathcal{N}'}$, if $f_{\mathcal{N}'}(\vx) = f_{\mathcal{N}}(\vx)$ for all $\vx \in \mathcal{X}$, we immediately conclude that $f_{\mathcal{N}'}(\vx) = f_{\mathcal{N}}(\vx)$ for all $\vx \in \bigcup_{j=1}^{md} \mathcal{L}_j$. By applying Lemma \ref{lem: tec_hyperplane}, we can further deduce that $\mathcal{N}' \sim \mathcal{N}$, thus completing the proof.

\end{proof}

 \section{Proof of Theorem \ref{th:gsamp}}\label{sec: sample_analytic}
 The following lemma is presented in  \cite[Lemma 1]{sussmann}. For the sake of convenience, we provide a straightforward alternative proof here.

\begin{lem}\cite[Lemma 1]{sussmann}\label{key_lemma} 
For each $j \in\{1, \ldots, m\}$, define $f_j(x) := \frac{1}{1 + \exp(-a_j \cdot x - b_j)}$ for some $a_j, b_j \in \mathbb{R}$. We assume that $(a_{j_1}, b_{j_1}) \neq \pm (a_{j_2}, b_{j_2})$ for $j_1 \neq j_2$, and that $a_j \neq 0$ for all $j\in \{1, \ldots, m\}$. Let $c_0, c_1, \ldots, c_m \in \mathbb{R}$. If $c_1\cdot f_1(x) + c_2 \cdot f_2(x) + \ldots + c_m \cdot f_m(x) + c_0 = 0$ holds for all $x \in \mathbb{R}$, then it follows that $c_0 = c_1 = c_2 = \cdots = c_m = 0.$
\end{lem}
\begin{proof}

For each $j\in \{1,\ldots,m\}$, if $a_{j} < 0$, we can replace $f_{j}(x)$ with $1 - f_{j}(-x)$.
 Hence, without loss of generality, we assume that $a_j>0$,  for all $j\in\{1,\ldots,m\}$.

A key observation is that $\lim_{x \to +\infty} f_j(x) = 0$
 for all $j \in \{1, \ldots, m\}$, which necessarily implies that $c_0 = 0$. Consequently, it suffices to demonstrate that $c_1 = c_2  =  \cdots = c_m = 0$.

We prove the statement by induction on $m$. For the base case, when $m=1$, the conclusion holds. Assuming the inductive hypothesis that the conclusion holds for all $m \leq n-1$, we now consider the case where $m=n$. Without loss of generality, we assume that 
\[
\alpha_0:=a_1=a_2= \cdots=a_{n_0}<a_{n_0+1}\leq a_{n_0+2}\leq  \cdots \leq a_n.
\]
A simple observation  is that when $x$ is large enough, 
\begin{equation}\label{eq:zhan}
f_j(x)=1-\sum_{k=1}^\infty (-1)^k\cdot \exp(-k\cdot (a_j\cdot x+b_j)), \qquad j=1,\ldots,n.
\end{equation}
The series of functions on the right-hand side converges uniformly to $f_j(x)$ when $x$ is sufficiently large.
We still use $f_j$ to denote the infinite function series.
Set $A:=\cup_{j=1}^n\{k\cdot a_j:  k\in \mathbb{Z}_{+}\}$. Then for any $\alpha \in A$ with ${\alpha}\leq M$
the coefficient of $\exp(-\alpha \cdot x)$ in the series expansion of  $\sum_{j=1}^n c_j\cdot f_j(x)$  is $0$. 
Here, $M$ represents a sufficiently large constant, which will be determined later.
Given that $\alpha \leq M$, to determine the coefficient of $\exp(-\alpha\cdot  x)$, it suffices to consider a finite truncation of each $f_j$.
We assert the existence of an increasing positive integer sequence $(p_t)_{t=1}^\infty$
  satisfying the following condition:
\begin{equation}\label{eq:jiaokong}
\{p_t\cdot \alpha_0:t\in {\mathbb Z}_{\geq 1}\}\cap  (\cup_{j=n_0+1}^n \{k\cdot a_j:  k\in \mathbb{Z}_{+}\})=\emptyset.
\end{equation}
This ensures that the coefficient of $\exp(-p_t\cdot \alpha_0\cdot x)$ in the series expansion of $\sum_{j=1}^n c_j\cdot f_j(x)$ is precisely $-(-1)^{p_t}\cdot \sum_{j=1}^{n_0}c_j\cdot \exp(-p_t \cdot b_j)$.
Then we have
\begin{equation}\label{eq:xiangdeng}
\sum_{j=1}^{n_0}c_j\cdot \exp(-p_{t}\cdot b_j)=0,\qquad t=1,\ldots,n_0.
\end{equation}
At this juncture, we can precisely define our choice of $M$: we stipulate that $M$ be any value satisfying the inequality $M > p_{n_0}\cdot \alpha_0$. According to (\ref{eq:xiangdeng}), we find that $c_1 = \cdots = c_{n_0} = 0$. To prove this, we utilize the fact that $(a_{j_1}, b_{j_1}) \neq (a_{j_2}, b_{j_2})$ whenever $j_1 \neq j_2$, along with the nonsingularity of the coefficient matrix of equation (\ref{eq:xiangdeng}). This property arises from the fact that $\{\exp(-p_t \cdot x)\}_{t=1}^{n_0}$  forms a Chebyshev system \cite{chebyshev}. 
By applying the inductive hypothesis, we can deduce the desired result.

It remains to establish the validity of (\ref{eq:jiaokong}). We proceed by contradiction. Suppose, contrary to our claim, that there exists $N_0 > 0$
  such that
\begin{equation*}
\{k\cdot \alpha_0: k\geq N_0\}\subseteq  \left(\cup_{j=n_0+1}^n \{k\cdot a_j:  k\in \mathbb{Z}_{+}\}\right).
\end{equation*}

  A fundamental theorem in number theory asserts the existence of infinitely many prime numbers exceeding any given positive integer. 
  Applying this to our context, and noting that $n$ is finite while there are infinitely many primes, we can, without loss of generality, select two distinct primes $q_1$
  and $q_2$, both greater than $N_0$, and an index $j_0\in \{n_0+1,\ldots,n\}$ satisfying the equations:
\begin{equation*}
 q_1\cdot \alpha_0 = q_1'\cdot a_{j_0} \quad \text{and} \quad q_2\cdot \alpha_0 = q_2'\cdot a_{j_0}
 \end{equation*}
for some positive integers $q_1'$  and $q_2'$. This leads to the equality $q_1'/q_1 = q_2'/q_2$. Therefore, we have $q_1' \geq q_1$  and $q_2' \geq q_2$ since $q_1$ and $q_2$ are primes and mutually distinct. However, given that $\alpha_0 < a_{j_0}$, we can deduce that $q_1' < q_1$ and $q_2' < q_2$. This results in a contradiction.
\end{proof}

Based on Lemma \ref{key_lemma}, we establish the following lemma, which plays a pivotal role in the proof of Theorem \ref{th:gsamp}.

\begin{lem}
\label{lem: sampling}
Take $f:\mathbb{R}\rightarrow \mathbb{R}$ as  
\[
f(x) := \sum_{k=1}^{n} s_k\cdot   \sigma(a_k \cdot  x+b_k ) +s_0,\]
 where  $\sigma(x)=\mathrm{Sigmoid}(x)$ or $\sigma(x)=\tanh(x)$.  Assume that the pairs $(a_k,b_k)\in \mathbb{R}\times \mathbb{R}$, $k=1,\ldots,n$, are mutually distinct, and $a_k\neq 0$, $k=1,\ldots,n$. 
 Then if $f$ possesses $2^n$ distinct zero points, it follows that $s_0 = s_1=\cdots=s_n= 0$.
\end{lem}
\begin{proof}
If $\sigma(x)=\tanh(x)$, owing to the identity $\tanh(x) \equiv 2\text{Sigmoid}(2x) - 1$, $f$ can be rewritten as:
\begin{equation}\label{eqn: g_new}
f(x)=\sum_{k=1}^n 2s_k \cdot \text{Sigmoid}\left(2a_k \cdot x+2b_k\right)+s_0-\sum_{k=1}^n s_k.
\end{equation}
Since  the pairs $(a_k,b_k)$, $k=1,\ldots,n$, are mutually distinct,   it follows that  the pairs $(2a_k,2b_k)$, $k=1,\ldots,n$, are also mutually distinct. Consequently, when the conclusion holds for $\sigma(x) = \text{Sigmoid}(x)$, we can deduce that $2s_k = 0$, for $k = 1, \ldots, n$, and $s_0 - \sum_{k=1}^n {s_k}= 0$ based on the expression (\ref{eqn: g_new}). This further implies that $s_k = 0$, for $k = 0, \ldots, n$. Therefore, we need only focus on the case where $\sigma(x) = \text{Sigmoid}(x)$. 
  
   Take $h(x):=f(x) \cdot  \prod_{k=1}^n (1+\exp(-(a_{k} \cdot x+b_k)))$, and then $h$ can also be expressed as 
 \begin{equation}\label{eqn: h}
\begin{aligned}
h(x)=&\sum_{k=1}^n s_k \cdot \prod_{j\neq k}(1+\exp(-(a_{j} \cdot x+b_j)))+s_0 \cdot  \prod_{k=1}^n (1+\exp(-(a_{k} \cdot  x+b_k))).
\end{aligned}
\end{equation}
Since $\prod_{k=1}^n (1+\exp(-(a_{k}\cdot  x+b_k)))>0$  for all $x \in \mathbb{R}$, the zero sets of $f$ and $h$ are identical.

Set
\begin{equation}\label{eqn: A}
{A}:=\Big\{\sum_{k\in K} {a_k}\ : \ K\subseteq\{1,\ldots,n\}\Big\},
\end{equation}
where we adopt the convention that the sum over an empty set is zero, i.e., $\sum_{k\in \emptyset} {a_k} = 0$.
 Therefore,  $h$ can be rewritten as 
\begin{equation}\label{eqn: h_new}
\begin{aligned}
h(x)=\sum_{\alpha\in {A}} c_{\alpha}\cdot   \exp(-\alpha \cdot   x),
\end{aligned}
\end{equation}
  where the coefficients $c_{\alpha}$ are given by
\begin{equation}\label{eqn: c_alpha}
c_{\alpha}=\sum_{K\in \mathcal{K}_{\alpha}}\Bigg(\Big(\sum_{k\in \{0,\ldots,n\}\setminus K} s_k\Big)\cdot  \Big(\prod_{k\in K}\exp(-b_k)\Big)\Bigg)
\end{equation}
with 
\[
\mathcal{K}_{\alpha}:=\{K\subseteq\{1,\ldots,n\}\ :\ \sum_{k\in K}a_k=\alpha\}.
\]
 We adopt the convention that $\prod_{k\in K}\exp(-b_k)=1$, if $K=\emptyset$.

From the definition of ${A}$ in (\ref{eqn: A}), we deduce that $\#{A} \leq 2^n$. Consequently, the set $\{\exp(-\alpha   x)\}_{\alpha \in {A}}$ forms a Chebyshev system with at most $2^n$  elements. Since the zero sets of $f$ and $h$ coincide, the fact that $f$ has $2^n$  distinct zeros implies that $h$ also has $2^n$ distinct zeros. This leads to the conclusion that $c_\alpha = 0$ for all $\alpha \in {A}$, and therefore $h \equiv 0$. From this, it follows that $f \equiv 0$. By applying Lemma \ref{key_lemma} and utilizing the linear independence of $1, \sigma(a_1\cdot x + b_1), \dots, \sigma(a_n\cdot x + b_n)$, along with $f \equiv 0$, we immediately conclude that $s_0 = s_1  =  \cdots = s_n = 0$.
\end{proof}
 
We will now introduce the definition of a {\em full spark frame} (see \cite{full}), which plays a crucial role in constructing the finite set of sampling points in Theorem \ref{th:gsamp}.

\begin{defi}
Let $\cV:=\{\vv_1,\ldots,\vv_N\}\subset \R^{d}$. We define $\cV$ to be a full spark frame if and only if every subset of $d$ vectors from $\cV$, denoted as $\{\vv_{j_1},\ldots,\vv_{j_d}\}$, spans the entire space $\R^d$, or equivalently,
\begin{equation*}
 \text{\rm span}\{\vv_{j_1},\ldots,\vv_{j_d}\} = \R^d.
 \end{equation*}
\end{defi}
 
Given positive integers $d$ and $N$ such that $d \leq N$, a variety of construction methods for full spark frames have been proposed in the literature (see \cite{full}). Notable among these are the constructions derived from structured matrices, such as the Vandermonde matrix, as well as those based on probabilistic approaches, exemplified by random Gaussian matrices.

 \begin{lem}\label{le:fullspark}
Let $\cV = \{\vv_1, \ldots, \vv_N\} \subseteq \mathbb{R}^{d}$ be a full spark frame. For any set of $M$ mutually distinct vectors $\{\va_1, \ldots, \va_M\}\subseteq\mathbb{R}^{d}$, if $N \geq \binom{M}{2}\cdot (d-1) + 1$, then there exists a vector $\vv_{j_0} \in \cV$ such that for all indices $k_1, k_2$  with $1 \leq k_1 < k_2 \leq M$, we have $\inner{\va_{k_1}, \vv_{j_0}} \neq \inner{\va_{k_2}, \vv_{j_0}}$.
 \end{lem}
 \begin{proof}
 We proceed by contradiction. For the sake of contradiction, assume that for any $\boldsymbol{v} \in \cV$, there exist indices $1 \leq k_1 < k_2 \leq m$ such that:
\[
\langle\boldsymbol{a}_{k_1}, \boldsymbol{v}\rangle = \langle\boldsymbol{a}_{k_2}, \boldsymbol{v}\rangle.
\]
For $1 \leq k_1 < k_2 \leq M$, define:
\[
I_{k_1,k_2} := \{\boldsymbol{v} \in \cV: \langle\boldsymbol{a}_{k_1}, \boldsymbol{v}\rangle = \langle\boldsymbol{a}_{k_2}, \boldsymbol{v}\rangle\}.
\]
Then we have:
\[
\bigcup_{1 \leq k_1< k_2 \leq M} I_{k_1,k_2} = \{\boldsymbol{v}_1, \ldots, \boldsymbol{v}_N\}.
\]
This leads to the claim that there exist indices $1 \leq k_1 < k_2 \leq M$ such that $\#I_{k_1,k_2} \geq d$. To prove this by contradiction, let us assume that $\#I_{k_1,k_2} \leq d-1$ for all pairs $(k_1,k_2)$ where $1 \leq k_1< k_2 \leq m$. Under this assumption, we would arrive at the following inequality:
\[
N = \#\{\boldsymbol{v}_1, \ldots, \boldsymbol{v}_N\} \leq \sum_{1 \leq k_1< k_2 \leq M} \#I_{k_1,k_2} \leq \binom{M}{2}\cdot (d-1).
\]
This inequality contradicts  $N \geq \binom{M}{2}(d-1) + 1$, thereby proving our claim.

Without loss of generality, we may assume that $\# I_{1,2}\geq d$ and $\{\boldsymbol{v}_1, \ldots, \boldsymbol{v}_d\} \subseteq I_{1,2}$. Then:
\[
\langle\boldsymbol{a}_{1}, \boldsymbol{v}_j\rangle = \langle\boldsymbol{a}_{2}, \boldsymbol{v}_j\rangle, \quad \text{for}\  j = 1, \ldots, d.
\]
According to ${\rm span}\{\vv_1,\ldots,\vv_d\}=\R^d$, we have  $\boldsymbol{a}_{1} = \boldsymbol{a}_{2}$, which contradicts our initial assumption that the vectors $\va_1,\ldots,\va_M$ are mutually distinct.
 \end{proof}

 
\begin{proof}[Proof of Theorem \ref{th:gsamp}]
Recall that 
\begin{equation}\label{eqn:2fn}
 f_{\mathcal{N}}(\vx)=\sum_{k=1}^m s_k\cdot  \sigma(\langle \va_k,\vx\rangle+b_k)+c\qquad \text{and}\qquad f_{\mathcal{N}'}(\vx)=\sum_{k=1}^m s_k'\cdot   \sigma(\langle \va_k',\vx\rangle+b'_k)+c'.
\end{equation}
Without loss of generality, we may assume that the first non-zero entry in each ${\va}_k$
  and ${\va}'_k$  is positive for all $k \in \{1, \ldots, m\}$. This assumption is justified by the identity $\sigma(x) + \sigma(-x) \equiv c_0$, which demonstrates that we can always transform the first negative entries into positive ones by adjusting the inputs of $\sigma$ and the constant terms.

Let us consider the finite point set $\mathcal{X} \subset \mathbb{R}^d$
constructed as follows:
\begin{enumerate}[(i)]
\item Set $N := \binom{4m}{2}\cdot (d-1)+1$.
Choose $\cV := \{\boldsymbol{v}_1, \ldots, \boldsymbol{v}_N\} \subseteq \mathbb{R}^d$
such that $\cV$ is a full spark frame in $\mathbb{R}^d$.
\item Select $2^{2m}$ distinct real numbers $z_1, z_2, \ldots, z_{2^{2m}} \in \mathbb{R}$.
\item For each $i \in\{ 1, \ldots, 2^{2m}\}$ and $j \in \{1, \ldots, N\}$, define: 
 \begin{equation}\label{eqn: x_kl}
  \boldsymbol{x}_{i,j}  := z_i\cdot   \boldsymbol{v}_j\in \mathbb{R}^d .
  \end{equation}
 \item Take $\mathcal{X} := \{  \boldsymbol{x}_{i,j} \}_{1 \leq i \leq 2^{2m},1\leq j\leq N} \subseteq \mathbb{R}^d$.
\end{enumerate}

We will proceed to demonstrate that, if
\begin{equation}\label{eqn: condition}
f_{\mathcal{N}}( \boldsymbol{x}_{i,j} ) = f_{\mathcal{N}'}( \boldsymbol{x}_{i,j} ),\quad  i=1,\ldots, 2^{2m}\text{ and } j=1,\ldots,N, 
\end{equation}
 then it follows that $\mathcal{N} \sim \mathcal{N}'$. 
 
 We assert  that the pairs $(\va_k,b_k), (\va'_k,b'_k)\in \mathbb{R}^{d}\times \mathbb{R}$, $k=1,\ldots,m$, satisfy 
 \begin{equation}\label{eq:jihedeng}
 \{(\va_k' ,b'_k)\}_{k=1}^m\,\,=\,\,\{(\va_k ,b_k)\}_{k=1}^m,
 \end{equation}
 and will provide the proof for this assertion later in our discussion.  Without loss of generality, assume that $(\va_k' ,b_k')=(\va_k ,b_k), k=1,\ldots,m$.
 Hence,   for any $\vx\in \R^d$, we have
 \begin{equation}\label{eq:FNN}
F(\vx):=F_{\mathcal N, \mathcal {N'}}(\vx):= f_{\mathcal{N}}(\vx)-f_{\mathcal{N}'}(\vx) =\sum_{k=1}^m (s_k-s_k') \cdot  \sigma(\langle \va_k,\vx\rangle+b_k)+c-c'.
 \end{equation}

Applying Lemma \ref{le:fullspark} to  the full spark frame $\cV$ in construction procedure (i) and the vector set $\{{\va}_1, \ldots, {\va}_m, -{\va}_1, \ldots, -{\va}_m, \boldsymbol{0}_d\}$ by removing duplicates, where $\boldsymbol{0}_d$ denotes the $d$-dimensional zero vector, we conclude that there exists a vector in $\cV$, let us call it ${\vv}_1$, such that:
\[
\begin{cases}
\langle  \va_{k_1},\vv_1 \rangle \neq \pm \langle  \va_{k_2},\vv_1 \rangle, & \text{for } \va_{k_1} \neq {\va}_{k_2}; \\
\langle  {\va}_k,\vv_1 \rangle \neq 0, & \text{for } k \in \{1, \ldots, m\}.
\end{cases}
\]
Combined with the fact that $ (\va_{k_1}, b_{k_1}) \neq \epsilon \cdot (\va_{k_2}, b_{k_2})$ for all $\epsilon \in \{-1,+1\}$ and all distinct indices $k_1, k_2$,  it leads to 
\begin{equation}\label{eqn: a_new}
\begin{cases}
(\inner{ \va_{k_1},\vv_1} , b_{k_1})\neq \pm (\inner{ \va_{k_2},\vv_1} , b_{k_2})\quad  & \text{for} \ k_1\neq k_2; \\
\langle  {\va}_k,\vv_1 \rangle \neq 0, & \text{for } k \in \{1, \ldots, m\}.
\end{cases}
\end{equation}
According to (\ref{eq:FNN}), for all $x\in \R$, we have
\[
F(x\cdot  \vv_1)=\sum_{k=1}^m (s_k-s_k')  \cdot  \sigma( \inner{ \va_k,\vv_1}\cdot  x+b_k)+c-c'.
\]
Based on (\ref{eqn: condition}), we observe that the function $F(x\cdot \vv_1)$,   with $x$ as its variable, has at least $2^m$ distinct zero points $z_i$, for $i = 1, \ldots, 2^m$. Coupled with (\ref{eqn: a_new}), this allows us to apply Lemma \ref{lem: sampling} to $F(x \cdot \vv_1)$, resulting in $c = c'$, and $s_k = s_k'$,  $k = 1, \ldots, m$. Consequently, we conclude that $\mathcal{N} \sim \mathcal{N}'$.

It remains to prove (\ref{eq:jihedeng}). We begin by showing that $ \{(\va_k' ,b'_k)\}_{k=1}^m \subseteq \{(\va_k ,b_k)\}_{k=1}^m$. For the sake of contradiction, assume that $ \{(\va_k' ,b'_k)\}_{k=1}^m \nsubseteq \{(\va_k ,b_k)\}_{k=1}^m$. Without loss of generality, we may assume that $(\va_1' ,b'_1)\notin \{(\va_k ,b_k)\}_{k=1}^m$. Here we also have $(\va_1' ,b'_1)\notin \{(-\va_k ,-b_k)\}_{k=1}^m$ as the first non-zero entry in each ${\va}_k$
  and ${\va}'_k$  is positive for all $k \in \{1, \ldots, m\}$.
According to Lemma \ref{le:fullspark}, employing a similar argument to the vector set (after removing the duplicates)
\[
\{{\va}_1, \ldots, {\va}_m, -{\va}_1, \ldots, -{\va}_m,  {\va}'_2, \ldots, {\va}'_m, -{\va}'_2, \ldots, -{\va}'_m, \va'_1,\boldsymbol{0}_d\},
\] 
there exists a vector in $\cV$, say ${\vv}_1$, such that:
\begin{equation}\label{eq:xishu}
\begin{cases}
\inner{\va_1',\vv_1} \neq  0, \\
(\inner{\va_1',\vv_1},b'_1)\neq \pm (\inner{\va_k,\vv_1},b_k), \quad &k=1,\ldots,m; \\
(\inner{\va_1',\vv_1},b'_1)\neq \pm (\inner{\va_k',\vv_1'},b'_k), \quad &k=2,\ldots,m .
\end{cases}
\end{equation}
Observe that for all $x\in \R$,
\[
F(x\cdot  \vv_1)=\sum_{k=1}^m s_k  \cdot \sigma(\inner{ \va_k,\vv_1} \cdot x+b_k)-\sum_{k=1}^m s_k' \cdot   \sigma(\inner{ \va'_k,\vv_1}\cdot x+b'_k)+c-c'.
\]
Based on (\ref{eqn: condition}), we establish that the function $F(x\cdot   \vv_1)$,  with $x$ as its variable, possesses at least $2^{2m}$ distinct zero points $z_i$, for $i=1,\ldots,2^{2m}$. Combined with (\ref{eq:xishu}), this enables use to apply Lemma \ref{lem: sampling} to deduce that $s_1'=0$, which contradicts the irreducibility of $f_{\mathcal{N}'}$. Consequently, we conclude that $ \{(\va_k' ,b'_k)\}_{k=1}^m \subseteq \{(\va_k ,b_k)\}_{k=1}^m$. 

Through an analogous argument, we can demonstrate that $ \{(\va_k ,b_k)\}_{k=1}^m \subseteq \{(\va_k' ,b_k')\}_{k=1}^m$. Therefore, we have established the validity of (\ref{eq:jihedeng}).

\end{proof}

\end{document}